\documentclass[10pt,journal,compsoc]{IEEEtran}
\ifCLASSOPTIONcompsoc
  \usepackage[nocompress]{cite}
\else
  \usepackage{cite}
\fi
\ifCLASSINFOpdf
   \usepackage[pdftex]{graphicx}
   \graphicspath{{../pdf/}{../jpeg/}}
\else
   \usepackage[dvips]{graphicx}
   \graphicspath{{../eps/}}
\fi
\usepackage{amsmath}
\usepackage{amssymb}
\usepackage{amsfonts}
\usepackage{fdsymbol}

\usepackage{algorithmic}
\usepackage{algorithm}
\makeatletter
\newcommand{\removelatexerror}{\let\@latex@error\@gobble}
\makeatother

\usepackage{array}

\ifCLASSOPTIONcompsoc
  \usepackage[caption=false,font=footnotesize,labelfont=sf,textfont=sf]{subfig}
\else
  \usepackage[caption=false,font=footnotesize]{subfig}
\fi
\usepackage{caption}

\usepackage{url}

\usepackage[english]{babel}
\usepackage[utf8]{inputenc} 
\usepackage[T1]{fontenc}    
\usepackage{hyperref}       
\usepackage{booktabs}       
\usepackage{nicefrac}       
\usepackage{microtype}      
\usepackage{graphicx}
\usepackage{amsthm}
\usepackage{color}
\usepackage{multirow}
\usepackage{float}
\usepackage{pythonhighlight}
\usepackage{amsmath}

\newtheorem{definition}{Definition}[section]

\newtheorem{theorem}{Theorem}[section]
\newtheorem{assumption}{Assumption}[section]

\begin{document}
%
\title{Modeling Multiple Views via Implicitly Preserving Global Consistency and Local Complementarity}
%
%
%
%

\author{Jiangmeng~Li,
        Wenwen~Qiang,
        Changwen~Zheng,
        Bing~Su,
        Farid~Razzak,
        Ji-Rong Wen,~\IEEEmembership{Senior Member,~IEEE}
        and~Hui~Xiong,~\IEEEmembership{Fellow,~IEEE}
\IEEEcompsocitemizethanks{\IEEEcompsocthanksitem J. Li and W. Qiang are with the University of Chinese Academy of Sciences, Beijing, China. They are also with the Science \& Technology on Integrated Information System Laboratory, Institute of Software Chinese Academy of Sciences, Beijing, China. E-mail: jiangmeng2019@iscas.ac.cn, a01114115@163.com. They contributed equally to this work.
\IEEEcompsocthanksitem C. Zheng is with the Science \& Technology on Integrated Information
System Laboratory, Institute of Software Chinese Academy of Sciences,
Beijing, China. E-mail: changwen@iscas.ac.cn.
\IEEEcompsocthanksitem F.Razzak is with the New York University \& Columbia University, New York, USA. E-mail: farid.razzak@nyu.edu.
\IEEEcompsocthanksitem B. Su and J.-R. Wen are with the Beijing Key Laboratory of Big Data Management and Analysis Methods, Gaoling School of Artificial Intelligence, Renmin University of China, Beijing, 100872, China. E-mail: subingats@gmail.com;  jrwen@ruc.edu.cn. Corresponding author: Bing Su.
\IEEEcompsocthanksitem H. Xiong is with Thrust of Artificial Intelligence, the Hong Kong University of Science and Technology (Guangzhou), Guangzhou, China. He is also with Department of Computer Science \& Engineering, the Hong Kong University of Science and Technology, Hong Kong SAR, China. E-mail: xionghui@ust.hk.
\IEEEcompsocthanksitem ©2021 IEEE. Personal use of this material is permitted. Permission from IEEE must be obtained for all other uses, in any current or future media, including reprinting/republishing this material for advertising or promotional purposes, creating new collective works, for resale or redistribution to servers or lists, or reuse of any copyrighted component of this work in other works.
}}
%
%

\markboth{SUBMITTED TO IEEE TRANSACTIONS ON KNOWLEDGE AND DATA ENGINEERING}%
{Shell \MakeLowercase{\textit{et al.}}: Bare Demo of IEEEtran.cls for Computer Society Journals}
%



\IEEEtitleabstractindextext{%
\begin{abstract}
While self-supervised learning techniques are often used to mine hidden knowledge from unlabeled data via modeling multiple views, it is unclear how to perform effective representation learning in a complex and inconsistent context. To this end, we propose a new multi-view self-supervised learning method, namely \textit{consistency and complementarity network} (CoCoNet), to comprehensively learn global inter-view consistent and local cross-view complementarity-preserving representations from multiple views. To capture crucial common knowledge which is implicitly shared among views, CoCoNet employs a global consistency module that aligns the probabilistic distribution of views by utilizing an efficient discrepancy metric based on the generalized sliced Wasserstein distance. To incorporate cross-view complementary information, CoCoNet proposes a heuristic complementarity-aware contrastive learning approach, which extracts a complementarity-factor jointing cross-view discriminative knowledge and uses it as the contrast to guide the learning of view-specific encoders. Theoretically, the superiority of CoCoNet is verified by our information-theoretical-based analyses. Empirically, our thorough experimental results show that CoCoNet outperforms the state-of-the-art self-supervised methods by a significant margin, for instance, CoCoNet beats the best benchmark method by an average margin of 1.1\% on ImageNet.

\end{abstract}

\begin{IEEEkeywords}
unsupervised learning, self-supervised learning, representation learning, multi-view, regularization, Wasserstein distance.
\end{IEEEkeywords}}

\maketitle

\IEEEdisplaynontitleabstractindextext

%
\IEEEpeerreviewmaketitle

\IEEEraisesectionheading{\section{Introduction}\label{sec:introduction}}
\begin{figure*}
	\vskip -0in
	\begin{center}
		\centerline{\includegraphics[width=1.95\columnwidth]{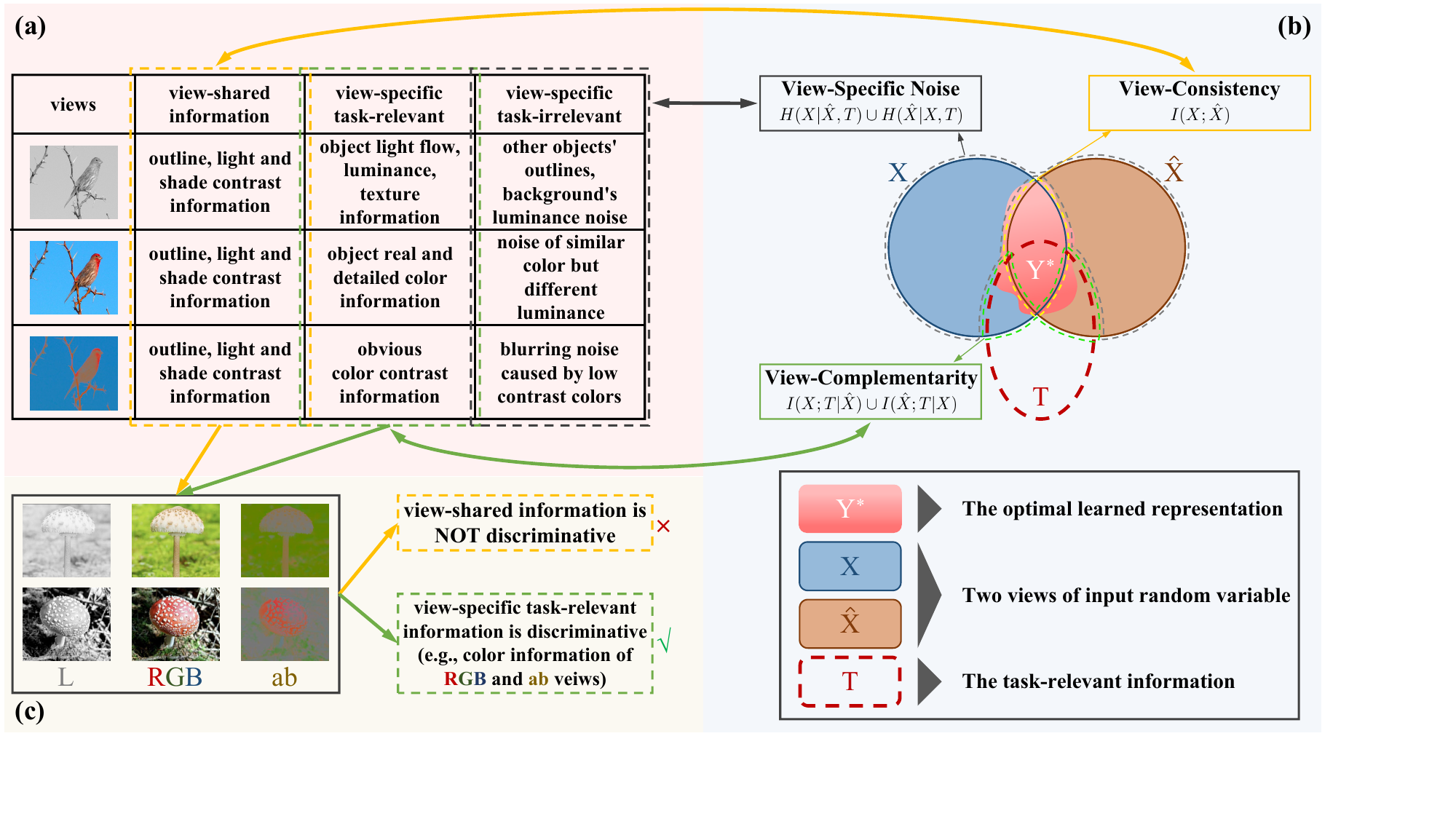}}
		\vskip -0.in
		\caption{Illustration of the theoretical analysis based on information theory, where the \textcolor[RGB]{255,143,00}{\textit{yellow}} boxes show the shared view-consistency information, the \textcolor[RGB]{66,99,00}{\textit{green}} boxes represent the view-specific and task-relevant information which is also the desired view-complementarity information, and the \textcolor[RGB]{99,99,99}{\textit{grey}} boxes denote the view-specific but task-irrelevant noise information. (a) An example of the three mentioned information in practical application; (b) the definitions of the mentioned information in the information-theoretical perspective; (c) The application of using multiple views to learn discriminative information on a specific image classification task, which proves that only mining view-shared information is \textit{not} enough to learn discriminative representations on benchmark datasets, e.g., ImageNet \cite{ffl09} has many similar fine-grained categories.}
		\label{fig:infotheory}
	\end{center}
	\vskip -0.35in
\end{figure*}

\IEEEPARstart{S}{elf-supervised} learning (SSL) aims to learn representations from unlabeled data that nonetheless can have wide-reaching benefits. The key to the problem lies in designing appropriate SSL objectives. Recent works explore how to maximize the mutual information (MI) between the inputs and outputs of the encoder. For example, the MI between high dimensional continuous random variables is effectively estimated by neural networks over gradient descent \cite{ib18}, and the MI between the high-level features and the local regions of the low-level features are jointly maximized \cite{rdh19}. The capacity of the encoder is crucial for estimating the MI between input-output pairs, and the ultimate goal of this approach is to encode the discriminative information for downstream tasks (e.g., classification) into representations. However, maximizing the MI between the input and the output of an encoder over a single view may encode the view-specific task-irrelevant information into the learned representations. These methods differ from the normal human learning process in how observations are represented. Humans have a tendency to perceive items from a variety of perspectives, including aural, gustatory, and visual. Concretely, the single-view-based methods are unable to extract task-relevant information from other views.

Recently proposed self-supervised learning methods, for example, SimCLR \cite{tc20}, SwAV \cite{cm20}, MoCo \cite{2020Km}, AMDIM \cite{pb19}, gLMSC \cite{cqz20}, and CMC \cite{ylt20}, have extended maximizing the MI between the encoder input and output on single-view data to maximizing the MI between the same samples under different views on multi-view data. As an example, given two views $X$ and $\hat X$ of image data, these methods concentrate on maximizing the MI $I({X};{\hat X})$. The assumption behind these methods is that the task-relevant information lies mostly in the shared information between the different views \cite{2008Sridharan}. Since the background of the data in different views may be different, maximizing the MI between the same data in different views will cause the encoder to focus on the shared information of the foreground in different views. It should be noted that for each view, the task-relevant discriminative information that is unique to that view also exists, which is referred to the view-specific and task-relevant information. In Figure \ref{fig:infotheory} (a), we show an example of such view-specific and task-relevant information in image data for classification. However, there are no additional terms in the objective of the benchmark methods to extract the view-specific and task-relevant information. In Figure \ref{fig:infotheory} (c), we further show an application to demonstrate that only mining view-shared information is not enough so that mining the view-specific and task-relevant information can improve the general discriminability of the learned representations.

According to information theory, the information contained in the input is divided into three parts. Figure \ref{fig:infotheory} (b) shows an example with a two-view dataset, where $X$ and $\hat X$ denote two views of a same sample, respectively, $T$ denotes task-relevant information or label-relevant information, and $Y^*$ denotes the optimal learned representation. The three parts of $X$ and $\hat X$ are as follows: the view-consistency information $I({X};{\hat X})$ to denote the view-shared information, which refers to the part surrounded by the yellow line; the view-complementarity information $I\left( {{X};T\left| {\hat X} \right.} \right)$ and $I\left( {{\hat{X}};T\left| {X} \right.} \right)$ to denote the view-specific task-relevant information, which refer to the part surrounded by the green lines; and the view-specific noise $H\left( {{X}\left| {{\hat X}} \right.,T} \right)$ and $H\left( {{\hat{X}}\left| {{X}} \right.,T} \right)$ to denote the view-specific task-irrelevant information, which refers to the part enclosed by the grey lines. Meanwhile, we give their formal definitions in Section \ref{sec:ta}. Therefore, we suppose the discriminative learned representation should contain both view-consistency and view-complementaity information and discard view-specific noise, i.e., $H\left( Y^* \right) = I\left( {X;\hat X} \right) + I\left( {X;T\left| {\hat X} \right.} \right) + I\left( {\hat X;T\left| X \right.} \right)$.

However, benchmark methods are difficult to achieve such a objective. We rethink the learning paradigms of conventional self-supervised multi-view learning methods from the perspective of information theory, which is demonstrated in Figure \ref{fig:concept}. As shown in Figure \ref{fig:concept} (a), methods that maximize the MI between the inputs and the outputs of the encoder over a single view aim to extract the task-relevant information contained in a single view, e.g., $I({X};{T})$, which refer to the red shaded part. However, such a built self-supervision problem is not enough to make the model to capture task-relevant information so that, after training, the optimal learned representation $Y^*$ may contain the view-specific noise, i.e., $H\left( {{X}, {{Y^*}} \left| \right.T} \right)$, which is denoted by the grey shaded part. Also, the task-relevant information contained in the other view, e.g., $I\left( {\hat X;T\left| X \right.} \right)$, can not be extracted. As demonstrated in Figure \ref{fig:concept} (b), the benchmark methods that maximize the MI between the different views of a same sample can only extract the view-consistency information contained in the $I({X};{\hat X})$ part and discard the view-specific noise $H\left( {{X}\left| \right.{{\hat X}}, T} \right)$ and $H\left( {{\hat X}\left| \right.{{X}}, T} \right)$. However, the view-complementaity information contained in $H\left( {{X}, T \left|\right.{{\hat X}}} \right)$ and $H\left( {{X}, T \left|\right.{{\hat X}}} \right)$ may also be discarded. Therefore, we motivate our method to sufficiently capture view-consistency and -complementarity information while discarding the view-specific noise, and the conceptual learning paradigm of our method is demonstrated in Figure \ref{fig:concept} (c).

\begin{figure}[t]
	\vskip -0.in
	\begin{center}
		\centerline{\includegraphics[width=0.99\columnwidth]{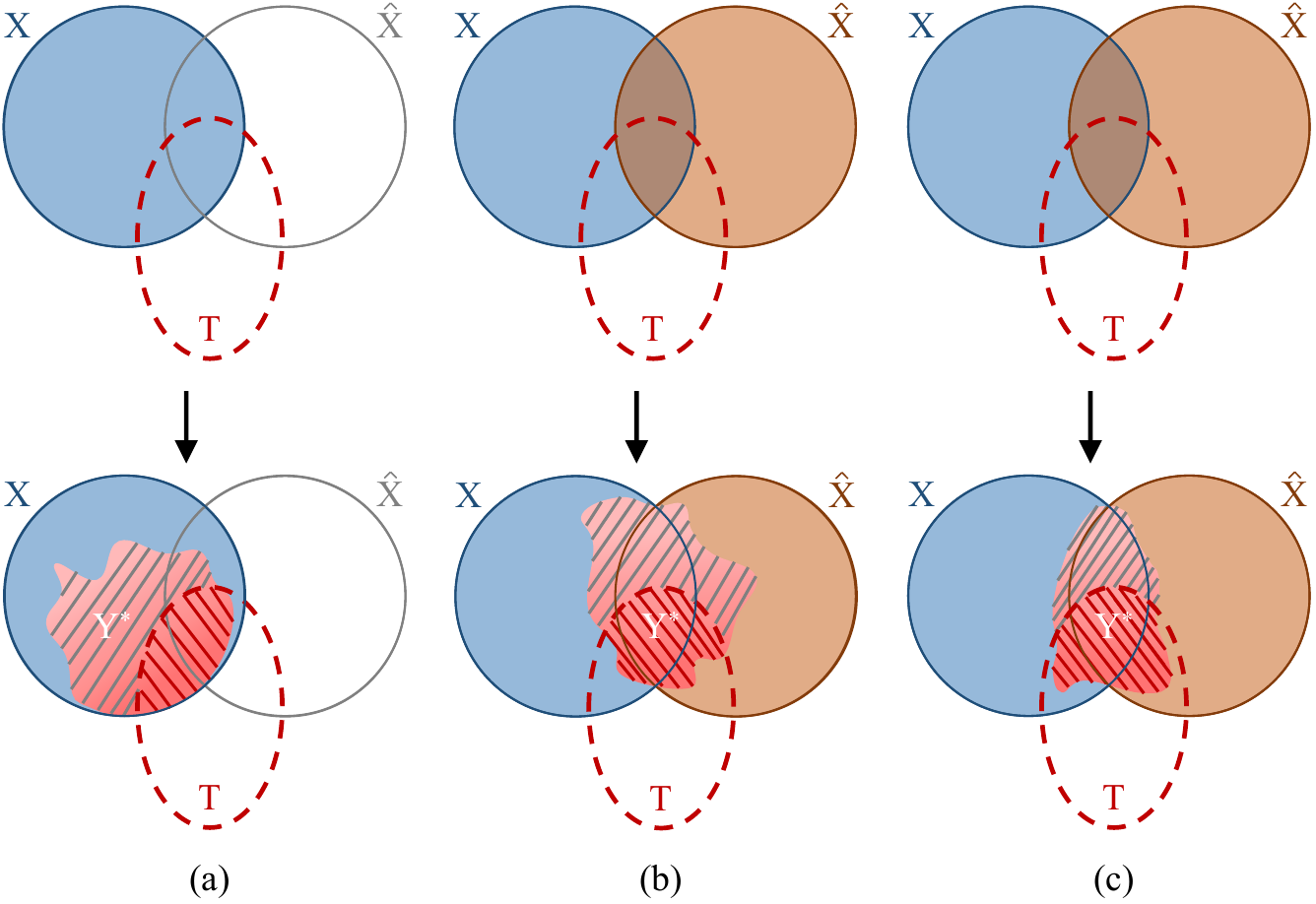}}
		\vskip -0.1in
		\caption{The conceptual learning paradigm plots of SSL methods. (a) the single-view SSL approach maximizing the MI between the inputs and the outputs of the encoder; (b) the conventional multi-view SSL approach maximizing the MI between the different views of a same sample; (c) our method. For the optimal learned representation $Y^*$, the \textcolor[RGB]{143,00,00}{\textit{red}} shaded part denotes the task-relevant information, and the \textcolor[RGB]{99,99,99}{\textit{grey}} shaded part denotes the task-irrelevant noise.}
		\label{fig:concept}
	\end{center}
	\vskip -0.35in
\end{figure}
To this end, we propose an integrated SSL method for modeling multi-view data called \textit{consistency and complementarity network} (CoCoNet). It projects all views into a latent space to obtain the feature representations and minimizes the generalized sliced Wasserstein distance discrepancy metric between the distribution of different views to enhance the consistency of multiple views in a global manner. For local single views, CoCoNet proposes a novel complementarity-aware contrastive learning approach, which leverages the complementarity factor to guide the encoders to capture the view-complementariry discriminative information and eliminate view-specific noise. In this way, CoCoNet aggregates the advances of multiple views and reduces the empirical risk of learning from each single view. Concretely, our proposed method aims to learn (albeit not fully) discriminative representations by using the strict consistency-preserving network to capture $I(X;\hat{X})$, and the proposed complementarity-aware contrastive learning approach prompts to capture $I(X;T|\hat{X})$ and $I(\hat{X};T|X)$. It is worth noting that the proposed CoCoNet is a general unsupervised representation learning model and the learned representation can be applied in a wide range of downstream tasks. In the experiments, we verify the effectiveness of CoCoNet on image classification tasks using multi-view data. The major contributions are four-fold:
\begin{itemize}
	\item We minimize a specific discrepancy metric to align the distributions of different views. As a result, the shared information between multiple views is extracted. This is to constrain our model to learn the view-consistency information, thereby reducing the impact of view-specific and task-irrelevant noise.
	\item We propose a heuristic complementarity-aware contrastive learning approach to enable the encoders to gain the view-specific and task-relevant information by using a novel complementarity-factor.
	\item We provide the information-theory-based analyses to demonstrate that preserving the global consistency and local complementarity can improve the discriminability of the learned multi-view representations.
	\item Following the protocol of \cite{ka19}, we perform empirical evaluations. Results show that CoCoNet outperforms previous works on benchmark and practical datasets. We have also demonstrated the generality of our method to different forms and types of multiple-view data with different characteristics.
\end{itemize}

\begin{figure*}
	\vskip -0in
	\begin{center}
		\centerline{\includegraphics[width=1.95\columnwidth]{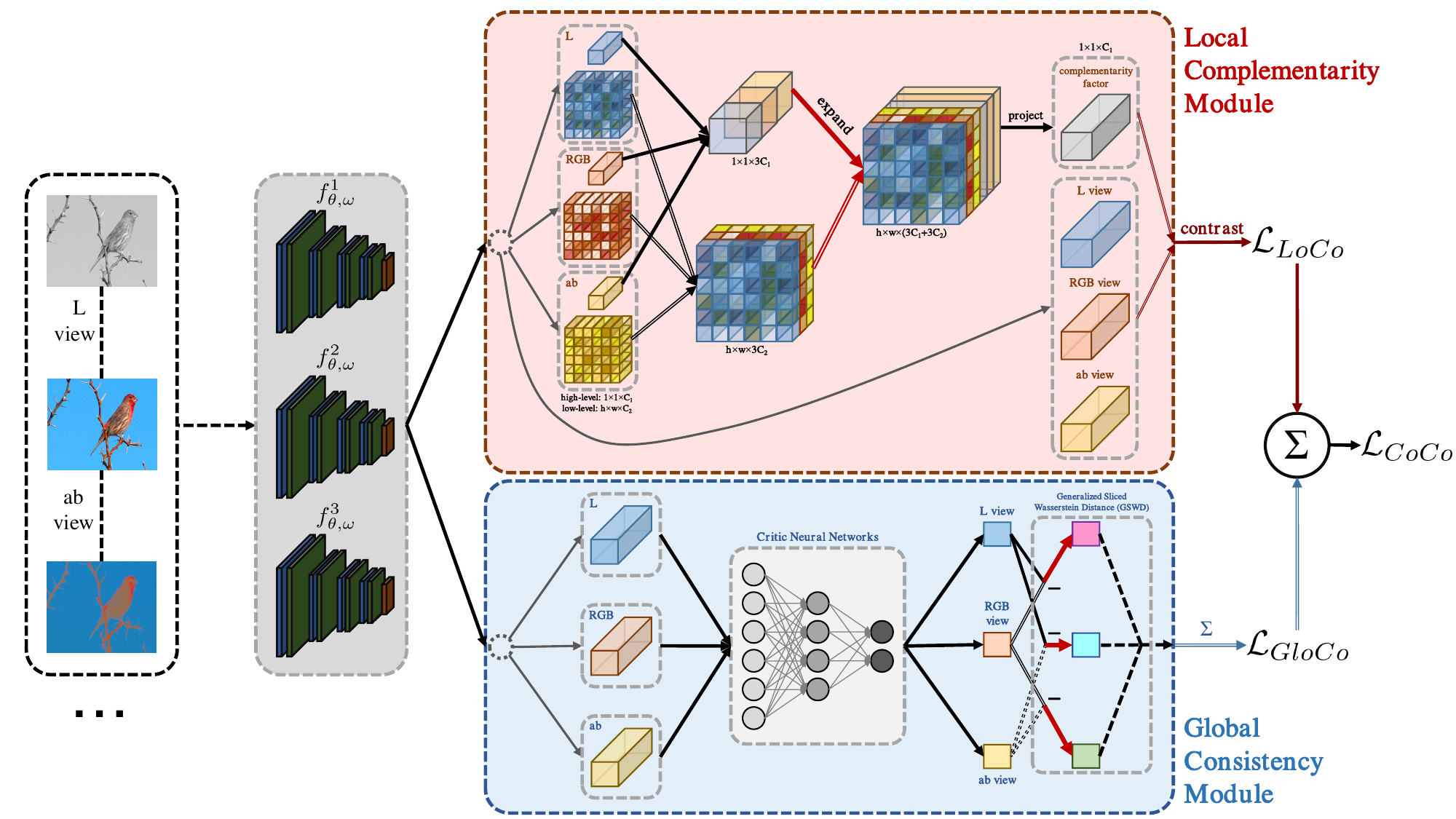}}
		\vskip -0.in
		\caption{Overview of the proposed CoCoNet, and we demonstrate an example of adopting three views: RGB view, L view, and ab view. It consists of two modules: the local complementarity module for capturing the view-complementarity information and the global consistency module for acquiring the view-consistency information, which can jointly eliminate view-specific noise and capture both of view-shared and view-specific discriminative information.}
		\label{fig:algoframe}
	\end{center}
	\vskip -0.35in
\end{figure*}

\section{Related works}
{\bf{Unsupervised learning.}}
Unsupervised representation learning gets rid of the reliance on labeled data \cite{bg13}. It starts with classical methods (without the use of deep neural networks), such as independent component analysis (ICA) \cite{bag95}, self-organizing maps \cite{kn98}, and principal components analysis (PCA) \cite{jit02}. SSL is a specific kind of unsupervised learning. However, there is a principal difference between it and classical unsupervised learning; it requires a designed generator of supervised learning problems. SSL methods must capture helpful information about the data to solve the generated problems.

SSL performs well with use of deep neural networks, and it started with seminal techniques (e.g., Boltzmann machines \cite{sp86, sr09}, autoencoders \cite{hge06}, variational autoencoders \cite{dpk14}, ${\beta}$ variational autoencoders \cite{aaa16}, generative adversarial networks \cite{gfi14}, adversarial autoencoders \cite{am15}, autoregressive models \cite{avd16}, BiGAN (a.k.a. adversarially learned inference with a deterministic encoder \cite{jd17}), Split-Brain Autoencoders (SplitBrain) \cite{rz17}, etc.). Recently, SSL is used in many fields, e.g., the NLP, vision, and robotics communities fields \cite{dj18, sp18}, and with the development of contrastive learning, several approaches based on it have come to the forefront, for instance, Noise As Targets (NAT) \cite{pb17}, Contrastive Predictive Coding (CPC) \cite{vdo18}, Deep InfoMax (DIM) \cite{rdh19}, a simple framework for contrastive learning of visual representations (SimCLR) \cite{tc20}, Momentum Contrast (MoCo) \cite{2020Km}, and Contrastive Multiview Coding (CMC) \cite{ylt20}. Existing generative models that maximize MI are also popular in this research area \cite{vp10, rdh19}. However, learning the representations from a single view does not significantly improve performance of the task.

{\bf{Multi-view learning.}}
In order to capture information from different views, existing Multi-View Learning methods jointly consider multiple views for different downstream tasks (e.g., clustering \cite{zyy19} and classification \cite{hz17}). The multi-view representation learning methods based on Canonical Correlation Analysis (CCA) \cite{hh35} are representative, which project different views into a common space, e.g., kernelized CCA \cite{sa06}, CCA-based deep neural network \cite{ww15}, and semi-pair and semi-supervised generalized correlation analysis (S2GCA) \cite{xhc12}. The unsupervised multi-view learning methods \cite{ylt20, rdh19, pb19} have also shown remarkable success in multi-view representation learning, while there is a crucial issue that learning from unaligned multiple views can lead to the poor performance of the representations.

{\bf{Distributions aligning.}}
Refer to the alignment of domain distributions, and the existing distribution alignment methods can be divided into three categories \cite{fzz21}. Instance-based approaches \cite{chen02} align the distributions by sub-sampling the training data of the two domains. Parameter-based approaches \cite{you2019} add adaptation layers or adaptive normalization layers to align the distributions. Furthermore, representation learning (RL) based approaches \cite{wu10, zhao09} primarily map the input to a common latent space and then align the two distributions in the latent space. Further, employing deep neural networks to align the distributions by minimizing a certain metric \cite{car01, sohn201} is a standard method found in RL-based approaches. The metrics of the RL-base approaches include the KL-divergence, the H-divergence, and the Wasserstein distance \cite{ben01, ma17, kuroki2019, DBLP:conf/nips/KolouriNSBR19}. The strength of Wasserstein distance, compared with other metrics, is that it takes advantage of gradient superiority. The literature findings motivated our proposal of a metric based on Wasserstein distance to align the distribution of different views in a latent space globally.
\section{Problem definition}
Formally, we consider the multi-view dataset ${X^m} = \left[{x_1^m,x_2^m,...,x_N^m} \right]$, where $X^m$ represents the sample collection from the $m$-th view, and ${x_i^m}$, $i=1,...,N$, $m=1,...,M$ denotes the $i$-th sample of the $m$-th view. $N$ is the number of samples in the $m$-th view and $M$ is the number of views. We denote $x^m$ as a random variable sampled \textit{i.i.d} from the distribution $\mathcal{P}\left(x^m \right) $. Also, we denote ${x_i} = \left[ {x_i^1;x_i^2; \cdots; x_i^M} \right]$, which represents a complete sample that consists of the same samples from different views, $X = \left[ {{x_1},{x_2},\cdots,{x_N}} \right]$ represents a complete dataset, and $x$ presents a variable sampled \textit{i.i.d} from distribution $\mathcal{P}\left(x \right) $. The self-supervised multi-view learning aims to learn multiple encoders capable of extracting discriminative features for each view’s data in an unsupervised manner, so that it can better serve downstream tasks such as classification. Specifically, we first project all multi-view data into the latent space through multiple encoders to obtain their feature representations, each view corresponds to a encoder, e.g., $f^m$ for the $m$-th view. Then, a certain objective is minimized to the parameters of the $f^m$. In this paper, the objective minimized in our proposed method consists of two parts including the loss function of the global consistency module and the complementarity-aware contrastive loss function. We will introduce these two loss function in the next section.
\section{Method}
In this section, we present the proposed \textit{consistency and complementarity network} (CoCoNet) in detail, which aims at learning feature representations that can jointly model view-consistent factors and view-complementary factors from multi-view data. As shown in Figure \ref{fig:algoframe}, CoCoNet consists of two modules, i.e., the local complementarity module for capturing the view-complementarity information, and the global consistency module for acquiring the view-consistency information.

\subsection{Global consistency module} \label{sec:gloco}
The idea behind global consistency module is to learn a feature representation that globally captures information shared among multiple views. We first adopt the view-specific encoders to project the original inputs $\left\{ X^1, X^2, \cdots, X^M \right\}$ into the latent space. The resulting latent representations $\left\{ H^1, H^2, \cdots, H^M \right\}$ fit the distributions $\mathcal{P}\left(H^1 \right) $, $\mathcal{P}\left(H^2 \right), \cdots, \mathcal{P}\left(H^M \right) $. Then, we align $\mathcal{P}\left(H^1 \right) $, $\mathcal{P}\left(H^2 \right), \cdots, \mathcal{P}\left(H^M \right) $ in the latent space by minimizing the discrepancy among the distributions.

Wasserstein distance is widely used as the discrepancy measure. Let ${\mathcal{P}}\left( H \right)$ be the set of Borel probability measures. For ${{P_r},{P_g} \in \mathcal{P}\left( H \right)}$ and the corresponding support set ${\Sigma _r,\Sigma _g}$, respectively. Then, the $p$-th Wasserstein distance of the corresponding distributions is defined as:
\begin{equation}
	\label{Eq:ad}
	\begin{aligned}
		{W_p}\left( {{P_r},{P_g}} \right) = {\left( {\mathop {\inf }\limits_{\mu \left( {x_r,x_g} \right) \in \Pi \left( {x_r,x_g} \right)} \int {c{{\left( {x_r,x_g} \right)}^p}d\mu } } \right)^{\frac{1}{p}}},
	\end{aligned}
\end{equation}
where $x_r \in {\Sigma _r}, x_g \in {\Sigma _g}$, $c\left( {x_r,x_g} \right)$ denotes the distance of two patterns in ${\Sigma _r}, {\Sigma _g}$, and $\Pi \left( {x_r,x_g} \right)$ denotes the set of all joint distributions $\mu \left( {x_r,x_g} \right)$ that satisfies ${P_r} = \int_{x_g} {\mu \left( {x_r,x_g} \right)d{x_g}}, {P_g} = \int_{x_r} {\mu \left( {x_r,x_g} \right)d{x_r}}$.

Directly calculating ${W_p}\left( p_1 ,p_2 \right)$ is computationally expensive. An alternative is using the popular dual version to calculate it, yet the Lipschitz constraint is difficult to meet. Therefore, we use the generalized sliced Wasserstein distances (GSWD), which is proposed by \cite{DBLP:conf/nips/KolouriNSBR19}, to approximate ${W_p}\left( p_1 ,p_2 \right)$. GSWD is defined as:
\begin{equation}
	GSWD_p\left( {{P_r},{P_g}} \right) = \int_{{\Omega_\vartheta}} {{W_p}} \left( {GR_\vartheta P_r , GR_\vartheta P_g} \right)d\vartheta
\end{equation}
where ${\Omega_\vartheta}$ denotes a compact set of feasible parameters for $ GR_\vartheta $, $GR_\vartheta$ represents one-dimensional nonlinear projection operation, also denoted as the critic neural network. Therefore, due to the non-linearity of $ GR_\vartheta $, the GSWD is expected to capture the complex structure of high-dimensional distributions (see the details of $ GR_\vartheta $ in Appendix \ref{sec:nonlmap}).

From the perspective of the gradient, we analyze the superiority of adopting GSWD as the discrepancy metric compared with other metrics, e.g., Kullback-Leibler (KL) divergence, which is described in Section \ref{sec:advgswd}. Empirically, we further explore the improvement of adopting various divergences as the discrepancy metric in Section \ref{sec:diffdis}.

Concretely, the loss function of the global consistency module is defined as:
\begin{equation} \label{equ:gloco}
	{\mathcal{L}_{GloCo}}= \sum\limits_{i=1}^{M-1} \sum\limits_{j=i+1}^M {GSWD_p}\left( {\mathcal{P}\left( {{H^i}} \right),\mathcal{P}\left( {{H^j}} \right)} \right)
\end{equation}

We denote the network by GloCo if only the global consistency preserving module is employed.

\subsection{Local complementarity module} \label{sec:loco}

Each view may contain unique discriminative information complementary to other views, which cannot be captured by the conventional contrastive learning approach. The local complementarity module aims to encode such view-complementarity information from multiple views in an instance-based manner. To this end, we first extract a complementarity-factor and then maximize the MI between this factor and the latent features of each view. The pipeline of the local complementarity module is depicted in Figure\ref{fig:algoframe}.

We incorporate the local discriminative knowledge of all views into a complementarity-factor. Specifically, given a sample $x_i$ with \(M\)-views $\left\{ {{x_i^1},{x_i^2},\cdots,{x_i^M}} \right\}$, we first encode these views to obtain the low-level feature maps $\left\{ {{z_i^1},{z_i^2},\cdots, {z_i^M}} \right\}$ with $h \times w \times C_2$ dimensions by the view-specific feature extraction networks $\left\{ f_\omega^1, f_\omega^2, \cdots, f_\omega^M \right\}$, where $C_2$, $h$, and $w$ are the number of channels, height, and width of the low-level feature maps, respectively. Then, we map $\left\{ {{z_i^1},{z_i^2},\cdots, {z_i^M}} \right\}$ into $C_1$-dimensional high-level feature vectors $\left\{ {{h_i^1},{h_i^2}, \cdots, {h_i^M}} \right\}$ by the view-specific mapping networks $\left\{ f_\theta^1, f_\theta^2, \cdots, f_\theta^M \right\}$. We concatenate these high-level feature vectors to obtain a $M \cdot C_1$-dimensional syncretic feature vector $h_i$, which is reckoned to combine the shared information among high-level feature vectors $\left\{ {{h_i^1},{h_i^2}, \cdots, {h_i^M}} \right\}$. Then, the low-level feature maps are also concatenated to obtain a $h \times w \times M \cdot C_2$-dimensional syncretic feature map $z_i$, which is considered to capture the shared low-level information from $\left\{ {{z_i^1},{z_i^2},\cdots, {z_i^M}} \right\}$.

For the sake of combining both high-level and low-level information, we expand $h_i$ to a $h \times w \times M \cdot C_1$ feature map, and then concatenate it with the syncretic low-level feature map $z_i$ to obtain a $h \times w \times M \cdot (C_1+C_2)$ embedding. Then, we project this embedding to a $C_1$-dimensional feature vector, called complementarity-factor $CF_i$. For a sample $x_i$, we finally obtain \(M\) high-level vectors $\left\{ {{h_i^1},{h_i^2},\cdots, {h_i^M}} \right\}$ and a complementarity-factor $CF_i$.

$CF_i$ is expected to encode comprehensive and complementary information from different views, but may also contain redundant view-specific noises. To filter out such redundant information while maintain useful complementary information, we perform contrast learning to maximize the MI between the high-level features of different views and the complementarity-factor. In a minibatch with $n$ samples $\left[ {{x_1},{x_2},\cdots,{x_n}} \right]$, for each sample $x_i$, we regard $\left\{ {{h_i^1},{h_i^2},\cdots,{h_i^M}} \right\}$ of all its views and the complementarity-factor $CF_i$ as the positive terms, and $\left\{ {{h_j^1},{h_j^2},\cdots,{h_j^M}} \right\}$ and the corresponding complementarity-factors $CF_j$ of the other samples as the negative terms where $j \in \left\{ {1,...,n} \right\} \cap j \ne i$. 

Conventional contrastive loss \cite{vdo18} can be formulated:

\begin{equation} \label{equ:cl}
	{\mathcal{L}_{contrast}} = -\mathop{\mathbb{E}}_{X^n} {\left[\log \frac{{{S_\tau }\left(p \right)}}{{{S_\tau }\left( p \right) + \sum\limits_{j = 1}^k {{{S_\tau }\left( n_j \right)} } }}\right]}
\end{equation}
where $S_\tau$ is a score function to measure the positive pairs and the negative pairs, $p$ denotes the positive pair, $n_j$ denotes the negative pair sampled, and $k$ denotes the number of the sampled negative pairs.

We insert the complementarity-factor into the contrastive loss to generate a novel complementarity-aware contrastive loss by reformulating the Equation \ref{equ:cl} as follows:

\begin{equation}\label{equ:lococlori}
	\begin{aligned}
		{\mathcal{L}_{cf-contrast}} = - \alpha \cdot \mathop{\mathbb{E}}_{\widetilde{X}^n} {\left[\log \frac{{{S_\tau }\left(\widetilde{p} 	\right)}}{{{S_\tau }\left( \widetilde{p} \right) + \sum\limits_{j = 1}^k {{{S_\tau }\left( \widetilde{n}_j \right)} } }}\right]} \\- \beta \cdot \mathop{\mathbb{E}}_{X^n} {\left[\log \frac{{{S_\tau }\left(p \right)}}{{{S_\tau }\left( p \right) + \sum\limits_{j = 1}^k {{{S_\tau }\left( n_j \right)} } }}\right] }
	\end{aligned}
\end{equation}
where, $\widetilde{p}$ denotes the positive pair, which consists of a positive term, i.e., a high-level feature of view $h_i^m$ of the selected sample $x_i$, and the complementarity-factor $CF_i$. Also, $\widetilde{n}_j$ denotes the negative pair of a negative term and the according complementarity-factor. In order to further study the impacts of the two parts, we excessively set two hyper-parameters, i.e., $\alpha$ and $\beta$, to balance the two terms.

More specifically, the positive pair and negative pair are constructed as follows. For the complementarity-aware contrastive term, we group one of the positive terms, $\left\{ {{h_i^1},{h_i^2},\cdots,{h_i^M}} \right\}$ and $CF_i$, to form a positive pair. A negative term, $\left\{ {{h_j^1},{h_j^2},\cdots,{h_j^M}} \right\}$ and $CF_j$, are bonded as a negative pair. Then, expending the Equation \ref{equ:lococlori}, we formalize the proposed complementarity-aware contrastive loss function as follows:

\begin{equation} \label{equ:loco}
	\begin{aligned}
		{\mathcal{L}_{LoCo}} = - \alpha\cdot\sum\limits_{i = 1}^n {\sum\limits_{m = 1}^3 } {\log \frac{{{S_\tau }\left( {h_i^m, CF_i} \right)}}{{{S_\tau }\left( {h_i^m, CF_i} \right) + \sum\limits_{j = 1 \cap j \ne i}^n} {{S_\tau }\left( {h_i^m, CF_{j}} \right)} }} \\ - \beta\cdot\sum\limits_{i = 1}^n {\mathop{\sum\limits_{m = 1}^3 }{\sum\limits_{t = 1 \cap t \ne m}^3 } {\log \frac{{{S_\tau }\left( {h_i^m,h_i^t} \right)}}{{{S_\tau }\left( {h_i^m,h_i^t} \right) + \sum\limits_{j = 1 \cap j \ne i}^n} \ {\sum\limits_{q= 1}^3} {{S_\tau }\left( {h_i^m,h_{j}^{q}} \right)} }} }
	\end{aligned}
\end{equation}

${S_\tau }$ is the contrastive feature measurement function, which is implemented as:

\begin{equation}
	{S_\tau }\left( {a,b} \right) = \exp \left( {\frac{{\left\langle {\left( a \right),\left( b \right)} \right\rangle }}{{\left\| {\left( a \right)} \right\| \cdot \left\| {\left( b \right)} \right\|}} \cdot \frac{1}{\tau }} \right)
\end{equation}
where $a$ and $b$ denote the input high-level feature vectors, $\left\langle {,} \right\rangle $ is the inner product operator, $\left\| {} \right\|$ is the \(L_2\)-norm, $\tau$ is the fixed temperature coefficient. We denote the local complementarity module as LoCo.

\subsection{Consistency and complementarity network}
As shown in Figure \ref{fig:algoframe}, our proposed CoCoNet incorporates the local consistency module and the global consistency module. Overall, the loss for the proposed CoCoNet is the weighted sum of the losses for the two modules:

\begin{equation}
	\label{equ:coco}
	\mathcal{L}_{CoCo}={\mathcal{L}_{LoCo}} + {\gamma \cdot \mathcal{L}_{GloCo}}
\end{equation}
where $\gamma$ is the coefficient that controls the balance between $\mathcal{L}_{LoCo}$ and $\mathcal{L}_{GloCo}$. By substituting Equation \ref{equ:gloco} and Equation \ref{equ:loco} into Equation \ref{equ:coco}, the objective is formulated as follows:

\begin{equation}
	\label{equ:cocoobj}
	\begin{aligned}
		&\mathop {min}\limits_{{f_\omega },{f_\theta }} \Bigg\{ - \alpha\sum\limits_{i = 1}^n {\sum\limits_{m = 1}^3 } {\log \frac{{{S_\tau }\left( {h_i^m, CF_i} \right)}}{{{S_\tau }\left( {h_i^m, CF_i} \right) + \sum\limits_{j = 1 \cap j \ne i}^n} {{S_\tau }\left( {h_i^m, CF_{j}} \right)} }} \\ &\underbrace{ - \beta\sum\limits_{i = 1}^n {\mathop{\sum\limits_{m = 1}^3 }{\sum\limits_{t = 1 \cap t \ne m}^3 } {\log \frac{{{S_\tau }\left( {h_i^m,h_i^t} \right)}}{{{S_\tau }\left( {h_i^m,h_i^t} \right) + \sum\limits_{j = 1 \cap j \ne i}^n} \ {\sum\limits_{q= 1}^3} {{S_\tau }\left( {h_i^m,h_{j}^{q}} \right)} }} }}_{LoCo} \\ &\qquad \quad \, + \gamma \underbrace{{ \sum\limits_{i=1}^2 \sum\limits_{j=i+1}^3 {GSWD_p}\left( {\mathcal{P}\left( {{H^i}} \right),\mathcal{P}\left( {{H^j}} \right)} \right)}}_{GloCo} \Bigg\}
	\end{aligned}
\end{equation}

Minimizing the first two terms of equation \ref{equ:cocoobj}, i.e., the loss of LoCo, can guide the multi-view representations to model more discriminative local view-complementarity information, and minimizing the last term of equation \ref{equ:cocoobj}, i.e., the loss of GloCo, can globally make the representations consistent. We conduct experiments to study the influence of $\alpha$, $\beta$, and $\gamma$, respectively, which is manifested in Section \ref{sec:deepgoing}.

Following \cite{ylt20}, we maintain a memory bank to store latent features for each training sample. In addition, we build an extra memory bank for efficiently retrieving complementarity-factor features on the fly. We elaborate the training pipeline in Algorithm \ref{alg:CoCoNet}, and the code is available at \url{https://github.com/jiangmengli/CoCoNet}.

\begin{algorithm}[t]
	\caption{CoCoNet}
	\label{alg:CoCoNet}
	\vskip 0.in
	\begin{algorithmic}
		\STATE {\bfseries Input:} Multi-view dataset ${X^m} = \left[{x_1^m,x_2^m,...,x_N^m} \right]$, minibatch size $n$, critic network training steps $s$, the learning rates ${\ell_\omega}$ and ${\ell_\theta}$ for the view-specific feature extractors ${f_\omega}$ and mapping networks ${f_\theta}$, the learning rate ${\ell_{critic}}$ for the critic network ${GR_\vartheta}$, and hyperparameters $\alpha$, $\beta$, $\gamma$.\\
		\STATE {\bf Initialize} ${\ell_\omega}$, ${\ell_\theta}$, ${\ell_{critic}}$, ${f_\omega}$, ${f_\theta}$, and ${GR_\vartheta}$.
		\REPEAT
		\STATE Sample minibatch $\left\{ {x_i} \right\}_{i = 1}^{n} \in {X^m}$.
		\FOR{$t=1$ {\bfseries to} $s$}
		\STATE \# Fix ${f_\omega}$ and ${f_\theta}$ \\ \# Train the critic network ${GR_\vartheta}$ to get \textit{max} GSWD
		\STATE ${\vartheta} \leftarrow {\vartheta} - {\ell_{critic}}\cdot \Delta_{\vartheta}{\left(-\mathcal{L}_{GloCo}\right)}$
		\ENDFOR
		\STATE \# Fix ${GR_\vartheta}$ and train ${f_\omega}$ and ${f_\theta}$
		\STATE ${\omega} \leftarrow {\omega} - {\ell_{\omega}}\cdot \Delta_{\omega}{\left(\mathcal{L}_{LoCo} + \gamma \cdot \mathcal{L}_{GloCo}\right)}$ \\ ${\theta} \leftarrow {\theta} - {\ell_{\theta}}\cdot \Delta_{\theta}{\left(\mathcal{L}_{LoCo} + \gamma \cdot \mathcal{L}_{GloCo}\right)}$ \\
		\UNTIL ${f_\omega}$, ${f_\theta}$ converge.
	\end{algorithmic}
	\vskip -0.in
\end{algorithm}

\section{Theoretical analyses} \label{sec:ta}

In this section, we analyze the proposed CoCoNet from the information-theoretical perspective, and we also provide the theoretical analysis about the advantages of using the generalized sliced Wasserstein distance as the selected metric.

\subsection{The information-theoretical analysis of CoCoNet}
{\bf{Notation.}} Figure \ref{fig:infotheory} demonstrates a visual illustration of CoCoNet by using information theoretical description. We regard the input random variable as $X$ and another view of $X$ as $\hat{X}$ in the figure, e.g., $X = X^1$ and $\hat{X} = X^2$. $T$ presents the downstream task-relevant information. $Y^*$ is the optimal representation learned from the deterministic encoder $f_{\vartheta, \omega }(\cdot)$, i.e., ${f_\vartheta }\left( {{f_\omega }\left(\cdot\right)}\right)$ that includes the feature extraction network ${f_\omega }$ and the mapping network ${f_\vartheta}$. For random variables $A$, $B$, and $C$, $H(A)$ denotes the entropy of $A$, and $H(A|B)$ denotes the conditional entropy of $H(A)-H(B)$. Accordingly, $I(A;B)$ presents the MI of $A$ and $B$, and $I(A;B|C)$ represents the conditional MI of $I(A;B) - H(C)$. 

To clarify the information diagrams of CoCoNet, we detail the definitions as follows:

\begin{definition}
	\label{def:1}
	View-Consistency information is the discriminative information that is shared among views.
\end{definition}

\begin{definition}
	\label{def:2}
	View-Complementarity information is the task-relevant information that is view-specific.
\end{definition}

\begin{definition}
	\label{def:3}
	View-Specific Noise is the task-irrelevant information that only exists in one specific view.
\end{definition}

Considering the Definitions \ref{def:1}, \ref{def:2}, and \ref{def:3}, we rewrite the common multi-view assumption \cite{2008Sridharan, 2013Xu} to describe multi-view learning between multiple views:

\begin{assumption}
	\label{ass:1}
	(Multi-view, rewriting Assumption 1 in work \cite{2008Sridharan}). The different views are approximately redundant to each other for the task-relevant information, based on \ref{def:2}, which is the View-Complementarity information, denoted as $\epsilon^{complementarity}$. For the View-Complementarity information of each view, we have $I(X^i;T|X^j) \leq \epsilon^{complementarity}$ with $i, j \in \{{1, ..., m}\} \cap i \ne j$.
\end{assumption}

Assumption \ref{ass:1} states that, for $\epsilon^{complementarity}$, when it is small, the task-relevant information mainly lies in the MI between the input and the self-supervised signal. Therefore, when the number of views, $m$, is not large, as $m$ increases, $I(X^i;T|\{X^j\}_{j=1 \cap j \ne i}^{m})$ gets more compressed, and the ratio of discriminative task-relevant information rise, since the constraints of the MI, $\{\max_{i,j = 1 \cap i \ne j}^m{I(X^i;X^j)}\}$ become stronger. Accordingly, the latent representation is more discriminative, which is supported by the view-vanishing experiment (See Section \ref{sec:deepgoing}).

\begin{definition}
	\label{def:4}
	(Consistent and Complementary Multi-view Representations for Self-supervision). Let $Y$ denotes the initially learned multi-view representation and $Y^*$ denotes the consistent and complementary multi-view representation with restricted view-shared and view-specific discriminative knowledge: $Y^*=\mathop{\text{argmax}} \limits_{Y}I(Y;\{X^i\}_{i=1}^m;T)$ s.t. $I(Y;\{X^i\}_{i=1}^m)$ is maximized.
\end{definition}

To learn the view-consistent and view-complementary multi-view representations $Y^*$ from the multiple views $\{X^i\}_{i=1}^m$, different from previous works that roughly maximize the MI $I(Y;X^1;X^2)$ in a multi-view manner by utilizing the conventional contrastive learning framework, we globally align the distribution of view to guide the encoders to model view-shared information by availing of the efficient generalized sliced Wasserstein distance, which, based on the Definitions \ref{def:1} and \ref{def:3}, is defined as:

\begin{theorem}
	\label{the:1}
	(View-Consistency information with a potential loss of View-Specific Noise $\epsilon^{noise}$). $Y^*$ is the sufficiently compressed latent representation, while $Y$ is the self-supervised representation with part of the view-specific but task-irrelevant information $\epsilon^{noise}$. Formally, only considering two views, i.e., $X^1$ and $X^2$, $I(X^1;Y) \geq I(X^1;Y|\epsilon^{noise}) = I(X^1;Y^*) = I(X^1;X^2)$.
\end{theorem}

\begin{proof}
	See Proof \ref{pro:1} in Appendix \ref{sec:proof} for details.
\end{proof}

Based on Theorem \ref{the:1}, we propose an implicit View-Consistency preserving regularization, which is approximated by the global consistency network in Section \ref{sec:gloco}.

\begin{theorem}
	\label{the:2}
	(View-Complementarity information, which is view-specific and task-relevant). $Y$ is the sufficiently compressed latent representation, and $Y^*$ is the latent representation of adding view-specific and task-relevant information into $Y$. Formally, considering $X^1$ and $X^2$, $I(X^1;Y;T) = I(X^2;Y;T) = I(Y;T) \leq I(Y;T) + I(X^1;Y^*;T|X^2) + I(X^2;Y^*;T|X^1) = I(Y^*;T)$.
\end{theorem}

\begin{proof}
	See Proof \ref{pro:2} in Appendix \ref{sec:proof} for details.
\end{proof}

By the same token, with the intuition of Theorem \ref{the:2}, we introduce an implicit View-Complementarity preserving regularization and implement it by the local complementarity network in Section \ref{sec:loco}.


\subsection{The gradient advantages of the generalized sliced Wasserstein distance} \label{sec:advgswd}
In order to align the distributions of views, we minimize the divergence between distributions \{$\mathcal{P}\left(X^1 \right), \mathcal{P}\left(X^2 \right), ..., \mathcal{P}\left(X^m \right)$\}. To this end, we map data into a common latent space and then measure the distance based on a specific discrepancy metric, which reduces the dimensionality of representations in latent space. The representation's wide distribution may exist throughout the latent space. Then, for the conventional discrepancy metric (e.g., KL-divergence), the data points located in a region where the probability of a certain distribution is extremely greater than other distributions have little contribution to the gradient with cross-entropy loss. At the same time, the generalized sliced Wasserstein distance can provide stable gradients for every data point. Learning from \cite{nh10, ma17}, we found that there will be a gradient vanishing problem if making data indistinguishable based on the conventional discrepancy metric in the case of the distributions has supports lying on low dimensional manifolds in the latent space. Also, we can get stable gradients by adopting the generalized sliced Wasserstein distance. Theoretically, consistent performance is achievable by using the generalized sliced Wasserstein distance.

\section{Experiments}

In this section, we compared the proposed method against a fully-supervised classifier similar to the Alexnet architecture and various benchmark unsupervised methods to evaluate the performance of CoCoNet. To comprehensively evaluate the performance of the propose method, we imposed CoCoNet on \textit{four} major downstream tasks: 1) benchmark image classification; 2) benchmark graph prediction; 3) benchmark action recognition; 4) practical object detection.

\subsection{Preparation}
We conducted experiments on neural network methods (i.e., the convolutional (conv) neural network-based method and the fully-connected (fc) network-based method) on benchmark datasets. Furthermore, we studied the performance of the ablation models of CoCoNet in conducted experiments, and we took CIFAR10 as the target dataset for the deepgoing exploration.

For setting the ablation study of CoCoNet, we compared with two main ablation models: GloCo, and LoCo. In details, CoCoNet refers to the complete model that considering the global consistency preserving, and the local complementarity preserving (i.e., $\alpha$ = 1, $\beta$ = 0.5, $\gamma$ = ${10^{ - 4}}$). GloCo refers to an ablation model of CoCoNet by removing the local complementarity preserving module (i.e., $\alpha$ = 0, $\beta$ = 0, $\gamma$ = ${10^{ - 4}}$). Since GloCo only employs the global consistency module, it can be treated as a view-alignment method. Therefore, GloCo can also be applied to conventional SSL methods. We combined GloCo and conventional self-supervised methods, e.g., GloCo+SwAV\cite{cm20} and GloCo+CMC\cite{ylt20}, to evaluate the performance of GloCo, where GloCo serves as a trimmer to align the feature distribution of different views. The features for each view are generated by these SSL methods. Similarly, LoCo refers to the model with only the local complementarity module (i.e., $\alpha$ = 1, $\beta$ = 0.5, $\gamma$ = 0).

\begin{figure}
	\vskip 0in
	\begin{center}
		\centerline{\includegraphics[width=1.0\columnwidth]{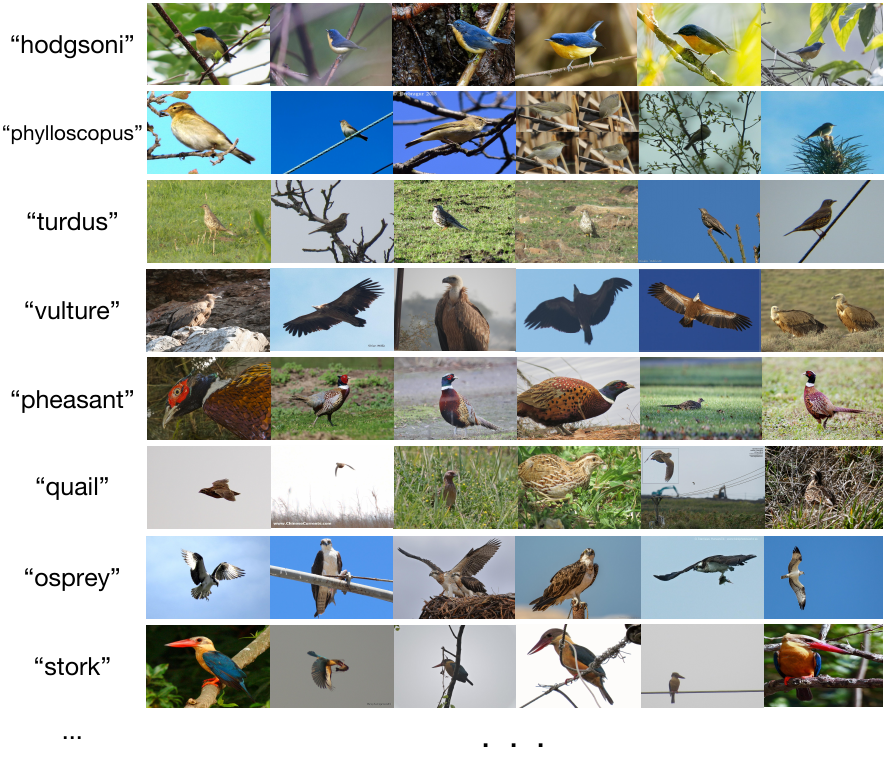}}
		\vskip -0.05in
		\caption{Example images in the WHEBD-759-SIM dataset.}
		\label{fig:whebdexample}
	\end{center}
	\vskip -0.37in
\end{figure}

\subsection{Datasets} \label{sec:datasets}

\subsubsection{Benchmark image classification datasets}
We conducted experiments on 5 established image representation learning datasets:

\textbf{CIFAR10 dataset} \cite{ka09} is a small-scale labeled dataset composed of 32 × 32 images with 10 classes. CIFAR10 has 45,000 examples for training the classifier, and 5,000 examples for testing.

\textbf{CIFAR100 dataset} \cite{ka09} is a small-scale labeled dataset consisting of 32 × 32 images of 100 categories, and each category contains 500 examples for training the classifier and 100 examples for testing.

\textbf{STL-10 dataset} \cite{ac11} is a dataset derived from ImageNet composed of 96 × 96 images, and it contains a mixture of 100,000 unlabeled training examples and 500 labeled examples per class.

\textbf{Tiny ImageNet dataset} \cite{ka09} is a reduced version of ImageNet ILSVRC \cite{ka12}, and images are scaled down to 64 × 64 with 200 classes. Each class has 500 images. The test set contains 10,000 images.

\textbf{ImageNet dataset} \cite{ffl09} consists of 1,000 image classes and is frequently considered as a testbed for unsupervised representation learning algorithms.

\subsubsection{Benchmark graph prediction datasets}
We conducted experiments on 5 established benchmark molecule prediction downstream tasks: molesol (mole), mollipo (moll), molbbbp (molb), moltox21 (molt) and molsider (mols) from \textbf{Open Graph Benchmark (OGB)} \cite{hu2020open}.

\subsubsection{Benchmark action recognition datasets}
We conducted experiments on 2 established video datasets:

\textbf{UCF-101 dataset} \cite{2012UCF101} is a dataset for realistic action videos, providing 13,320 videos from 101 action categories.

\textbf{HMDB-51 dataset} \cite{2011HMDB} contains 6,849 samples, divided into 51 categories, each category contains at least 101 samples.

\subsubsection{Practical object detection datasets}
We introduced CoCoNet on a practical bird detecting and classifying dataset to validate the effectiveness of the proposed method under real-world circumstances, as follows:


\textbf{Waterfowl and Habitat of Earth Big Data dataset (WHEBD)} is a real-world and extensible dataset, which is automatically updated in cycles. We truncated WHEBD-759-2020 (WHEBD-759) from the complete WHEBD. An example of WHEBD-759-SIM is demonstrated in Figure \ref{fig:whebdexample}, and the derived WHEBD-759 consists of 759 image classes and 699,815 samples in total.

\subsection{Implementations} \label{sec:implementation}

\subsubsection{Setup of benchmark image classification}
In the training, we used fixed hyper-parameters and a batch size of 128. In the test, we built conventional classifiers on the high-level vector representations extracted from the previous conv or fc network and then evaluated the performance of models by averaging the results of the last 40 epochs of optimizations. Here, we selected three views: the Red-Green-Blue (RGB) view of the original image, the luminance channel (L) view, and the ab-color channel (ab) view. Hence, in order to improve the discriminability of the learned features, we adopted a series of data augmentation methods: color jittering, random grayscale, and random cropping. We uniformly adopted the MI estimator based on Noise-Contrastive Estimation \cite{rdh19, ylt20}.

In the comparisons of Table \ref{tab:a}, \ref{tab:inclf}, and \ref{tab:f1measure}, the encoder function $f_{\vartheta, \omega }$, i.e., ${f_\vartheta }\left( {{f_\omega }\left(\cdot\right)}\right)$ that includes the feature extraction network ${f_\omega }$ and the mapping network ${f_\vartheta}$, is approximated by a designed Alexnet \cite{ka12} for the classification tasks. Inspired by the backbone splitting setting of SplitBrain \cite{rz17}, we evenly split the Alexnet into sub-networks across the channel dimension and then used the sub-networks as encoders. According to the principle of building the encoders, the Alexnet is split across the channel dimension with a conjecture that split-Alexnet can also perform well in learning representations between views, and the split-Alexnet only has the halved learnable parameters \cite{rz17}. We, therefore, built the Alexnet with five convolutional layers (attached with additional batchnorm layers, ReLU activation functions, and corresponding maxpool functions), two linear layers (with corresponding batchnorm layers and ReLU activation functions), and a fully connected layer followed by an l2 normalization function, which is to tackle the problem of distribution drift. Then the split-Alexnets (i.e., the sub-networks) are served as the encoders. In the experiments, we used the convolutional (conv) neural network and the fully-connected (fc) network as the encoders, which use respectively the layers of Alexnet as the encoders, i.e., conv has 5 convolutional layers and one fully connected layer, and fc is the complete Alexnet. For the sake of further exploring the influence of the network architecture on the performance of CoCoNet and the ablation models, we chose conv$n$ as the backbone network, where $n$ denotes the number of the convolutional layers. The comparisons on the large-scale ImageNet \cite{ffl09} are shown in Table \ref{tab:inclf}.

In addition, we conducted extended experiments on the CIFAR10 and STL-10 datasets, and Table \ref {tab:c} depicts the results. We followed the experimental settings of the CPC and DIM experiment \cite {rdh19}, and then a strode crop architecture \cite{vdo18} (i.e., eight × eight crops with four × four strides on the CIFAR10 dataset, and 16 × 16 crops with eight × eight strides on the STL-10 dataset) is adopted. We chose ResNet-50 \cite{kh16} architecture as the encoder $f_{\vartheta, \omega }$, and the same classifier as in Table \ref {tab:a} is used, which is conducted to study the performance of CoCoNet based on a different backbone network.

For a fair comparison, all benchmark datasets use backbone encoders without pretraining. In training, we held the perspective that the representations learned the crucial features of views through different encoders. Then, we directly concatenated representations layer-wise from the encoders into one to achieve the ultimate representation of an input sample. The classifier's development leverages a basic Multi-Layer Perception network (MLP) followed by the softmax output function. All downstream classification tasks are subject to the classifiers (i.e., the linear networks) on the high-level vector representations extracted from the designed encoders. For building the discrepancy metric calculation critic network based on generalized sliced Wasserstein distance (i.e., the critic network), the discrepancy metric of CoCoNet measures the differences between views in the learned latent space. Hence, the critic network is designed based on MLP \cite{1943mws}.

\subsubsection{Setup of benchmark graph prediction}
To evaluate our method on self-supervised graph classification and regression tasks. We combined CoCoNet with GraphCL \cite{you2020graph}, which is a benchmark graph contrastive learning method. In detail, given a graph $G = \{G_i | i \in N\}$, two augmented graphs $\widetilde{G}^1_i, \widetilde{G}^2_i: \widetilde{G}_i \thicksim aug\left(\widetilde{G}_i | G_i\right)$ are generated by random graph augmentations \cite{you2020graph, suresh2021adversarial}. $G_i$,  $\widetilde{G}^1_i$, and $\widetilde{G}^2_i$ are treated as three views of of a graph. We compared CoCoNet with 4 baselines, and the reason we selected such baselines is that the experimental results of InfoGraph \cite{sun2020infograph} and GraphCL \cite{you2020graph} show that they achieve the state-of-the-art and outperform graph kernel and network embedding approaches \cite{kriege2020survey, grover2016node2vec, adhikari2018sub2vec, yanardag2015deep, narayanan2017graph2vec, shervashidze2011weisfeiler}. We followed the experimental protocol of GraphCL and AD-GCL. The average classification accuracy and standard deviation of the test results over the 20 runs are reported in Table \ref{tab:graphprediction}. For a fair comparison, baselines and our methods adopt GIN as the encoder and use a downstream linear classifier or regressor with the same hyperparameters. To perform the ablation study, we constructed LoCo and two GloCo variants by combining GloCo with GraphCL and AD-GCL.

\subsubsection{Setup of benchmark action recognition}
We combined CoCoNet with CMC \cite{ylt20}, and then evaluate the performance of the combination variant on benchmark action recognization tasks by following the experimental setting of \cite{ylt20, 2007Christopher}. We trained our methods on UCF-101 \cite{2012UCF101} by using CaffeNets \cite{2017AlexImagenet} to learn features from images and optical flows. Two streams are applied in the method: 1) the ventral stream, which performs object recognition and connects the target frame (image) of a video stream with a neighbouring frame; 2) the dorsal stream, which processes motion and associates the target frame to optical flow (centered at the target frame) in video data. In the training, we adopted both ventral and dorsal streams, which can be treated as two views, and the target frame (image) in a video stream is the third view. In the test, the compared methods are tested on UCF-101 to evaluate the \textit{task} transferability and on HMDB-51 \cite{2011HMDB} to evaluate the \textit{task} and \textit{dataset} transferability. It is worthy to note that, we performed CoCoNet based on the reimplemented CMC, i.e., CMC$^\ast$.

\subsubsection{Setup of practical object detection}
In a real application, many newly added samples of CWD are unlabeled, and only a few samples are labeled per category on the procedure. Therefore, we introduced the proposed self-supervised method to enhance the performance of the main detection and classification models, e.g., B-CNN \cite{2015Lin} and Faster RCNN \cite{2017Ren}, in a semi-supervised manner.

For the experimental settings, we trained the main models by fully utilizing the labeled samples of WHEBD-759 as the control group (the results are manifested in the first 2 rows of Table \ref{tab:whebd}), which is denoted as $X=\{(x_1, y_1), (x_2, y_2), ..., (x_{|X|}, y_{|X|})\}$. In order to simulate the real scene of this application, we generated a general simulated dataset, i.e., WHEBD-759-SIM, from the original labeled dataset, where only a quarter of the labeled data is retained for each category, and the labels of the remaining data are discarded. Hence, WHEBD-759-SIM consists of a labeled set $L=\{(x^L_1, y^L_1), (x^L_2, y^L_2), ..., (x^L_{|L|}, y^L_{|L|})\}$, and a unlabeled set $U=\{x^U_1, x^U_2, ..., x^U_{|U|}\}$. We trained plain main models on WHEBD-759-SIM as another compared group. For the sake of taking advantage of the unlabeled data of $U$, we introduced the SSL methods, as the auxiliary methods, into the main models to form the integral semi-supervised learning methods. In details, the alternative SSL methods includes our proposed CoCoNet and the state-of-the-art self-supervised methods, for instance, SimCLR \cite{tc20}, SwAV \cite{cm20}, and CMC \cite{ylt20}. The supervised methods contains B-CNN \cite{2015Lin} and Faster-RCNN \cite{2017Ren}, and both of VGG-16 \cite{2014Ks} and VGG-19 \cite{2014Ks} are selected as the alternative encoders. We used the self-supervised methods to pretrain the backbone networks and then leveraged the supervised methods to train the model.

\begin{table*}[t]
	\vskip 0.05in
	\caption{Performance of top-1 classification accuracy (\%) on the CIFAR10, CIFAR100, Tiny ImageNet, and STL-10 datasets. In the experiments, we evaluated CoCoNet and the ablation models. Fully-supervised classification results are provided for comparison. $^\ddagger$ indicates that the results are reproduced by our reimplementation. For a fair comparison, we adopted the same backbone networks with benchmarks. Note that the results of SplitBrain and CMC on STL-10 are reported in \cite{ylt20}.}
	\vskip -0.15in
	\label{tab:a}
	\setlength{\tabcolsep}{6pt}
	\begin{center}
		\begin{tabular}{l|ccc||ccc||ccc||ccc}
			\hline\rule{-2.2pt}{8pt}
			\multirow{2}*{Model} & \multicolumn{3}{c||}{CIFAR10} & \multicolumn{3}{c||}{CIFAR100} & \multicolumn{3}{c||}{Tiny ImageNet} & \multicolumn{3}{c}{STL-10} \\ 
			\cline{2-13}\rule{-2.2pt}{8pt}
			& conv & fc & Average & conv & fc & Average & conv & fc & Average & conv & fc & Average \\
			\hline
			\hline\rule{-2.2pt}{8pt}
			\text{Fully supervised} & & 75.39 & & & 42.27 & & & 36.60 & & & 68.70& \\
			\hline\rule{-2.2pt}{8pt}
			\text{VAE\cite{dpk14}} & 60.71 & 60.54& 60.63 & 37.21 & 34.05 & 35.63 & 18.63 & 16.88 & 17.76 & 58.27 & 56.72 & 57.50 \\\rule{-2.2pt}{8pt}
			\text{AE\cite{hge06}} & 62.19 & 55.78 & 58.99 & 31.50 & 23.89 & 27.70 & 19.07 & 16.39 & 17.73 & 58.19 & 55.57 & 56.88 \\\rule{-2.2pt}{8pt}
			\text{${\beta}$-VAE\cite{aaa16}} & 62.40 & 57.89 & 60.15 & 32.28 & 26.89 & 29.59 & 19.29 & 16.77 & 18.03 & 57.15 & 55.14 & 56.15 \\\rule{-2.2pt}{8pt}
			\text{AAE\cite{am15}} & 59.44 & 57.19 & 58.32 & 36.22 & 33.38 & 34.80 & 18.04 & 17.27 & 17.66 & 59.54 & 54.47 & 57.01 \\\rule{-2.2pt}{8pt}
			\text{BiGAN\cite{jd17}} & 62.57 & 62.74 & 62.66 & 37.59 & 33.34 & 35.47 & 24.38 & 20.21 & 22.30 & 71.53 & 67.18 & 69.36 \\\rule{-2.2pt}{8pt}
			\text{NAT\cite{pb17}} & 56.19 & 51.29 & 53.74 & 29.18 & 24.57 & 26.88 & 13.70 & 11.62 & 12.66 & 64.32 & 61.43 & 62.88 \\\rule{-2.2pt}{8pt}
			\text{SplitBrain$^\ddagger$\cite{rz17}} & 77.56 & 76.80 & 77.18 & 51.74 & 47.02 & 49.38 & 32.95 & 33.24 & 33.10& 72.35 & 63.15 & 67.75 \\\rule{-2.2pt}{8pt}
			\text{DIM\cite{rdh19}} & 73.25 & 73.62 & 73.44 & 48.13 & 45.92 & 47.03 & 33.54 & 36.88 & 35.21 & 72.86 & 70.85 & 71.86 \\\rule{-2.2pt}{8pt}
			\text{SimCLR$^\ddagger$\cite{tc20}} & 80.58 & 80.07 & 80.33 & 50.03 & 49.82 & 49.93 & 36.24 & 39.83 & 38.04 & 75.57 & 77.15 & 76.36 \\\rule{-2.2pt}{8pt}
			\text{SwAV$^\ddagger$\cite{cm20}} & 66.18 & 69.23 & 67.71 & 50.87 & 51.23 & 51.05 & 39.56 & 38.87 & 39.22 & 70.32 & 71.40 & 70.86 \\\rule{-2.2pt}{8pt}
			\text{CMC$^\ddagger$\cite{ylt20}} & 81.31 & \textbf{83.28} & 82.30 & 58.13 & 56.72 & 57.43 & 41.58 & 40.11 & 40.85 & 83.03 & \textbf{85.06} & 84.05 \\
			\hline
			\hline\rule{-2.2pt}{8pt}
			\textbf{GloCo+SwAV} & 74.63 & 73.58 & 74.11 & 57.09 & 55.21 & 56.15 & 40.20 & 41.02 & 40.61 & 72.38 & 71.06 & 71.72 \\\rule{-2.2pt}{8pt}
			\textbf{GloCo+CMC} & 82.27 & 82.95 & 82.61 & \textbf{59.02} & 57.38 & 58.20 & 42.21 & 39.62 & 40.92 & 84.12 & 85.03 & \textbf{84.58} \\\rule{-2.2pt}{8pt}
			\textbf{LoCo} & 82.74 & 82.31 & 82.53 & 57.86 & \textbf{58.29} & 58.08 & \textbf{42.74} & 40.94 & 41.84 & 82.63 & 83.75 & 83.19 \\
			\hline\rule{-2.2pt}{9pt}
			\textbf{CoCoNet} & \textbf{83.10} & 83.24 & \textbf{83.17} & 58.64 & 58.21 & \textbf{58.43} & 42.28 & \textbf{43.63} & \textbf{42.96} & \textbf{85.34} & 83.82 & \textbf{84.58} \\
			\hline
		\end{tabular}
	\end{center}
	\vskip -0.2in
\end{table*}

\begin{table}[t]
	\vskip 0.in
	\caption{Performance of top-1 classification accuracy (\%) on the ImageNet dataset. We followed the experiments of CMC \cite{ylt20}, and we further reimplemented SimCLR and SwAV based on the same backbone networks.}
	\vskip -0.15in
	\label{tab:inclf}
	\setlength{\tabcolsep}{3pt}
	\begin{center}
		\begin{tabular}{l|cccccc}
			\hline\rule{-2.2pt}{8pt}
			\multirow{2}*{Model} & \multicolumn{6}{c}{ImageNet} \\ 
			\cline{2-7}\rule{-2.2pt}{8pt}
			& conv1 & conv2 & conv3 & conv4 & conv5 & Average \\
			\hline
			\hline\rule{-2.2pt}{8pt}
			\text{Fully supervised} & 19.3 & 36.3 & 44.2 & 48.3 & 50.5 & 39.7 \\
			\hline\rule{-2.2pt}{8pt}
			\text{Context \cite{2015Carl}} & 16.2 & 23.3 & 30.2 & 31.7 & 29.6 & 26.2 \\\rule{-2.2pt}{8pt}
			\text{Colorization \cite{2016Zhang}} & 13.1 & 24.8 & 31.0 & 32.6 & 31.8 & 26.7 \\\rule{-2.2pt}{8pt}
			\text{Jigsaw \cite{2016Noroozi}} & 19.2 & 30.1 & 34.7 & 33.9 & 28.3 & 29.2 \\\rule{-2.2pt}{8pt}
			\text{BiGAN \cite{jd17}} & 17.7 & 24.5 & 31.0 & 29.9 & 28.0 & 26.2 \\\rule{-2.2pt}{8pt}
			\text{SplitBrain \cite{rz17}} & 17.7 & 29.3 & 35.4 & 35.2 & 32.8 & 28.7 \\\rule{-2.2pt}{8pt}
			\text{Counting \cite{2017Noroozi}} & 18.0 & 30.6 & 34.3 & 32.5 & 25.7 & 28.2 \\\rule{-2.2pt}{8pt}
			\text{Inst-Dis \cite{un2}} & 16.8 & 26.5 & 31.8 & 34.1 & 35.6 & 29.0 \\\rule{-2.2pt}{8pt}
			\text{RotNet \cite{2018Gidaris}} & 18.8 & 31.7 & 38.7 & 38.2 & 36.5 & 32.8 \\\rule{-2.2pt}{8pt}
			\text{DeepCluster \cite{2018Caron}} & 12.9 & 29.2 & 38.2 & 39.8 & 36.1 & 32.2 \\\rule{-2.2pt}{8pt}
			\text{DIM \cite{rdh19}} & 14.5 & 24.9 & 29.1 & 32.4 & 35.9 & 27.4 \\\rule{-2.2pt}{8pt}
			\text{SimCLR$^\ddagger$\cite{tc20}} & 15.9 & 22.4 & 34.5 & 34.0 & 37.7 & 28.9 \\\rule{-2.2pt}{8pt}
			\text{SwAV$^\ddagger$\cite{cm20}} & 13.6 & 23.8 & 32.2 & 27.3 & 38.0 & 27.0 \\\rule{-2.2pt}{8pt}
			\text{CMC \cite{ylt20}} & \textbf{18.4} & 33.5 & 38.1 & 40.4 & 42.6 & 34.6 \\
			\hline
			\hline\rule{-2.2pt}{8pt}
			\textbf{GloCo+SwAV} & 16.8 & 28.5 & 33.7 & 26.2 & 34.9 & 28.0 \\\rule{-2.2pt}{8pt}
			\textbf{GloCo+CMC} & 17.7 & 36.2 & 39.6 & \textbf{41.1} & 43.0 & 35.5 \\\rule{-2.2pt}{8pt}
			\textbf{LoCo} & 17.9 & 34.4 & 38.4 & 38.6 & 43.7 & 34.6 \\
			\hline\rule{-2.2pt}{9pt}
			\textbf{CoCoNet} & 18.2 & \textbf{36.3} & \textbf{39.8} & 40.5 & \textbf{43.8} & \textbf{35.7} \\
			\hline
		\end{tabular}
	\end{center}
	\vskip -0.2in
\end{table}

\subsection{Results and discussion} \label{sec:results}
\subsubsection{Comparisons on benchmark image classification}
We extensively evaluated our proposed CoCoNet method on several benchmark datasets and tasks against the state-of-the-art methods. Table \ref{tab:a} shows the comparison results on the CIFAR10, CIFAR100, Tiny ImageNet, and STL-10 benchmark datasets respectively. The last 4 rows of tables represent the results of our proposed methods. Specifically, GloCo+SwAV, GloCo+CMC, and LoCo are the ablation models designed to eliminate different parts' influence. In general, CoCoNet outperforms all models presented here by a significant margin when using the benchmark datasets. CoCoNet even outperforms the fully-supervised classifier without fine-tuning for the specific architectures presented, which shows that the representations learned by CoCoNet are better than the original images. However, in different experimental settings, we found that a designed fully-supervised classifier can outperform the state-of-the-art methods by a wider margin. Meanwhile, when more powerful backbone networks are used as encoders and specific data augmentations are adopted, the approaches perform better on the benchmark datasets (albeit in different settings, e.g., AMDIM \cite{pb19}). However, these approaches leverage different and deeper networks as their backbone encoders, so we excluded these methods from the benchmarks. Hence, the ablation models, i.e., GloCo+SwAV, GloCo+CMC, and LoCo, outperform most state-of-the-art approaches on all datasets. Yet, our proposed methods only outperform CMC with a small advantage. After comparison, we found out that CMC adopts a specialized architecture with carefully-chosen data augmentations, and in general, our proposed GloCo can additionally enhance CMC, e.g., GloCo+CMC outperforms CMC. To our knowledge, in the field of unsupervised learning, the results of CoCoNet are state-of-the-art following the proposed experimental settings. Specifically, the results support the proposed CoCoNet effectiveness to preserve the consistency of unlabeled data across views. As shown in Tables \ref{tab:a}, the best results are in the last row, indicating that the feature representations learned by CoCoNet are discriminative.

\begin{table}[t]
	\vskip 0.in
	\caption{Performance of top-1 classification accuracy (\%) on the CIFAR10 and STL-10 datasets. We compared the proposed method with the state-of-the-art unsupervised methods. We adopted ResNet-50 \cite{kh16} as the encoders.}
	\vskip -0.15in
	\label{tab:c}
	\setlength{\tabcolsep}{10.5pt}
	\begin{center}
		\begin{tabular}{l|c|c|c}
			\hline\rule{-2.2pt}{8pt}
			\text{Model} & CIFAR10 & STL-10 & Average \\ 
			\hline
			\hline\rule{-2.2pt}{8pt}
			\text{CPC \cite{vdo18}} & 77.45 & 77.81 & 77.63 \\\rule{-2.2pt}{8pt}
			\text{DIM \cite{rdh19}} & 77.51 & 78.21 & 77.86 \\\rule{-2.2pt}{8pt}
			\text{SwAV$^\ddagger$ \cite{cm20}} & 83.15 & 82.93 & 83.04 \\\rule{-2.2pt}{8pt}
			\text{SimCLR$^\ddagger$ \cite{tc20}} & 84.63 & 83.75 & 84.19 \\\rule{-2.2pt}{8pt}
			\text{CMC$^{\ddagger}$ \cite{ylt20}} & 86.10 & 86.83 & 86.47 \\
			\hline
			\hline\rule{-2.2pt}{8pt}
			\textbf{GloCo+SwAV} & 84.62 & 85.81 & 85.22 \\\rule{-2.2pt}{8pt}
			\textbf{GloCo+CMC} & 87.78 & 89.11 & 88.45  \\\rule{-2.2pt}{8pt}
			\textbf{LoCo} & 89.06 & 88.21 & 88.64 \\
			\hline\rule{-2.2pt}{9pt}
			\textbf{CoCoNet} & \textbf{89.58} & \textbf{89.37} & \textbf{89.48} \\
			\hline
		\end{tabular}
	\end{center}
	\vskip -0.2in
\end{table}

\textbf{Benchmarking CoCoNet on a large-scale dataset.} As demonstrated in Table \ref{tab:inclf}, the proposed CoCoNet has consistent performance even on a large benchmark dataset (e.g., ImageNet) within different network architectures. CoCoNet beats the state-of-the-art unsupervised method (e.g., CMC) by 1.1\% on average. The performance of LoCo is on par with that of CMC, which demonstrates that our proposed local complementarity preserving module can improve the discriminability of the learned features by utilizing the complementarity-factor to capture the complementary information from different views. Furthermore, the ablation model GloCo+SwAV outperforms SwAV by 1.8\%, and GloCo+CMC outperforms CMC by 0.9\% respectively, which indicates that the global consistency preserving network enhances the baseline methods by aligning the distribution of views in the hidden space. We also observed that the performance of the compared methods is unstable within weaker backbone networks, such as conv1 or conv2, and our consideration behind this phenomenon is that oversimplified networks do not have enough mapping capabilities to learn discriminative high-dimensional feature representations by utilizing complex self-supervised tasks.

\begin{table*}[t]
	\vskip 0.in
	\renewcommand\arraystretch{1.2}
	\caption{Performance of classification (F1-Measure) on the CIFAR10, CIFAR100, STL-10, Tiny ImageNet, and ImageNet.}
	\vskip -0.15in
	\label{tab:f1measure}
	\setlength{\tabcolsep}{10pt}
	\begin{center}
		\begin{tabular}{l|cc||cc||cc||cc||c}
			\hline\rule{0pt}{9pt}
			\multirow{2}*{Model} & \multicolumn{2}{c||}{CIFAR10} & \multicolumn{2}{c||}{CIFAR100} & \multicolumn{2}{c||}{Tiny ImageNet} & \multicolumn{2}{c||}{STL-10} & \multicolumn{1}{c}{ImageNet} \\
			\cline{2-10}
			& conv & fc & conv & fc & conv & fc & conv & fc & conv \\
			\hline
			\hline
			\text{DIM \cite{rdh19}} & 0.7280 & 0.7276 & 0.4729 & 0.4435 & 0.3308 & 0.3461 & 0.7264 & 0.7042 & 0.3210 \\
			\text{SwAV \cite{cm20}} & 0.6578 & 0.6842 & 0.4990 & 0.4932 & 0.3808 & 0.3642 & 0.7000 & 0.7076 & 0.3407 \\
			\text{SimCLR \cite{tc20}} & 0.7980 & 0.7937 & 0.4856 & 0.4810 & 0.3508 & 0.3919 & 0.7542 & 0.7681 & 0.3445 \\
			\text{CMC \cite{ylt20}} & 0.8101 & 0.8210 & 0.5753 & 0.5611 & 0.3952 & 0.3858 & 0.8280 & \textbf{0.8461} & 0.3918 \\
			\hline
			\hline
			\textbf{GloCo + SwAV} & 0.7409 & 0.7324 & 0.5567 & 0.5338 & 0.3858 & 0.4056 & 0.7223 & 0.7039 & 0.3372 \\
			\textbf{LoCo} & 0.8232 & 0.8203 & 0.5609 & \textbf{0.5746} & 0.3996 & 0.3983 & 0.8234 & 0.8308 & \textbf{0.4085} \\
			\textbf{CoCoNet} & \textbf{0.8258} & \textbf{0.8241} & \textbf{0.5786} & 0.5745 & \textbf{0.4122} & \textbf{0.4239} & \textbf{0.8512} & 0.8347 & 0.4062 \\
			\hline
		\end{tabular}
	\end{center}
	\vskip -0.2in
\end{table*}

\begin{table}[t]
	\small
	\renewcommand\arraystretch{1.1}
	\vskip 0.05in
	\caption{Performance of chemical molecules property prediction in OGB datasets, including two downstream tasks: graph regression and graph classification.}
	\vskip -0.18in
	\label{tab:graphprediction}
	\setlength{\tabcolsep}{5pt}
	\begin{center}
		\begin{tabular}{l|cc||ccc}
			\hline\rule{0pt}{10pt}
			\multirow{3}*{Model}  & mole & moll  &  molb   & molt &  mols \\
			\cline{2-6}\rule{0pt}{10pt}
			& \multicolumn{2}{c||}{Regression} & \multicolumn{3}{c}{Classification} \\
			\cline{2-6}\rule{0pt}{10pt}
			& \multicolumn{2}{c||}{(RMSE $\downarrow$)} & \multicolumn{3}{c}{(ROC-AUC\% $\uparrow$)} \\			
			\hline
			\text{GIN RIU \cite{xu2018powerful}} & 1.706 & 1.075 & 64.48 & 71.53 & 62.29 \\
			\hline
			\text{InfoGraph \cite{sun2020infograph}} & 1.344 & 1.005 & 66.33 & 69.74 & 60.54 \\
			\text{GraphCL \cite{you2020graph}} & 1.272 & 0.910 & 68.22 & 72.40 & 61.76 \\
			\text{AD-GCL \cite{suresh2021adversarial}} & 1.270 & 0.926 & 68.26 & 71.08 & \textbf{61.83} \\
			\hline
			\textbf{GloCo + GraphCL} & 1.272 & 0.913 & 68.21 & \textbf{72.42} & 61.76 \\
			\textbf{GloCo + AD-GCL} & 1.271 & 0.925 & 68.24 & 71.09 & 61.78 \\
			\textbf{LoCo} & \textbf{1.268} & 0.910 & 68.49 & 71.32 & 61.81 \\
			\textbf{CoCoNet} & 1.269 & \textbf{0.907} & \textbf{68.53} & 72.37 & 61.80 \\
			\hline
		\end{tabular}
	\end{center}
	\vskip -0.2in
\end{table}

\textbf{Performing CoCoNet with ResNet.} We performed extended classification comparisons on the CIFAR10 and STL-10 datasets with results shown in Table \ref{tab:c}. We set the experiments by following the same principle of the CPC and DIM comparison \cite{rdh19}. The results show that CoCoNet and the experimental ablation models outperform state-of-the-art methods on the CIFAR10 and STL-10 datasets, respectively. Since ResNet-based encoders are adopted on the comparisons, and our proposed methods outperform the benchmarks, we reckoned that CoCoNet has strong adaptability to different encoders.

\textbf{Evaluation with F1-Measure.} In Table \ref{tab:f1measure}, we evaluated the compared methods with F1-Measure. We observed that CoCoNet is still state-of-the-art. It is widely acknowledged that there would be a big difference between the results of Accuracy and F1-Measure in the case of imbalanced sample categories (e.g., long-tail datasets). The benchmark datasets we conducted are all balanced datasets, including CIFAR10, CIFAR100, Tiny ImageNet, STL-10, and ImageNet. So, we considered that the results of our comparisons would not show a big difference, whether it is based on Accuracy or F1-Measure. However, in the case of imbalanced datasets, we have also included the F1-measure scores for the comparisons on the test sets of benchmark datasets to provide support for the superiority of our proposed CoCoNet, which is demonstrated in Table \ref{tab:f1measure}. Note that the comparisons are based on Macro F1-Measure, and the results indicate that CoCoNet can still achieve the best performance.

\begin{table}[t]
	\renewcommand\arraystretch{1.1}
	\vskip 0.05in
	\caption{Action recognition accuracy (\%) to evaluate \textit{task} and \textit{dataset} transferability on benchmark video datasets. We followed the setting of \cite{ylt20}. $\dagger$ denotes different network architecture. $\ast$ denotes our reimplementation.}
	\vskip -0.18in
	\label{tab:action}
	\setlength{\tabcolsep}{4.5pt}
	\begin{center}
		\begin{small}
			\begin{tabular}{l|c|c|c}
				\hline
				\multirow{2}*{Method} & Number & \multirow{2}*{UCF-101} & \multirow{2}*{HMDB-51} \\
				& of Views & & \\
				\hline
				\text{Random} & - & 48.2 & 19.5 \\
				\text{ImageNet} & - & 67.7 & 28.0 \\
				\hline
				\text{VGAN$^\dagger$ \cite{2016Generating}} & 2 & 52.1 & - \\
				\text{LT-Motion$^\dagger$ \cite{2017Unsupervised}} & 2 & 53.0 & - \\
				\hline
				\text{TempCoh \cite{DBLP:conf/icml/MobahiCW09}} & 1 & 45.4 & 15.9 \\
				\text{Shuffle and Learn \cite{DBLP:conf/eccv/MisraZH16}} & 1 & 50.2 & 18.1 \\
				\text{Geometry \cite{2018Geometry}} & 2 & 55.1 & 23.3 \\
				\text{OPN \cite{2017UnsupervisedLee}} & 1 & 56.3 & 22.1 \\
				\text{ST Order \cite{Uta2018Improving}} & 1 & 58.6 & 25.0 \\
				\text{Cross and Learn \cite{2018Cross}} & 2 & 58.7 & 27.2 \\
				\text{CMC \cite{ylt20}} & 3 & 59.1 & 26.7 \\
				\text{CMC$^\ast$} & 3 & 58.8 & 26.3 \\
				\hline
				\textbf{GloCo + CMC} & 3 & \textbf{59.5} & 27.0 \\
				\textbf{LoCo} & 3 & 59.2 & 26.8 \\
				\textbf{CoCoNet} & 3 & 59.4 & \textbf{27.4} \\
				\hline
			\end{tabular}
		\end{small}
	\end{center}
	\vskip -0.2in
\end{table}

\begin{table*}[t]
	\vskip 0.in
	\caption{Comparison of classification top-1 accuracies (\%) on the real-world WHEBD-759-SIM dataset. We incorporated various SSL methods, e.g., SimCLR \cite{tc20}, SwAV \cite{cm20}, CMC \cite{ylt20}, the proposed CoCoNet, and the ablation models of CoCoNet, into the main supervised detection and classification methods, e.g., B-CNN \cite{2015Lin}, and Faster-RCNN \cite{2017Ren}, to form the improved semi-supervised learning methods for the tasks, respectively. VGG-16 \cite{2014Ks} and VGG-19 \cite{2014Ks} are the backbone networks.}
	\vskip -0.15in
	\label{tab:whebd}
	\setlength{\tabcolsep}{13pt}
	\begin{center}
		\begin{tabular}{c|l|l|cc|c}
			\hline\rule{-2.2pt}{8pt}
			\multirow{2}*{Training set} & \multirow{2}*{Supervised model} & \multirow{2}*{Self-supervised model} & \multicolumn{2}{c|}{Backbone network} & \multirow{2}*{Average} \\
			\cline{4-5}\rule{-2.2pt}{9pt}
			& & & VGG-16 & VGG-19 & \\
			\hline
			\hline\rule{-2.2pt}{8pt}
			\multirow{2}*{WHEBD-759} & \text{B-CNN} & N/A & \textbf{87.7} & 86.2 & 87.0 \\
			\rule{-2.2pt}{8pt}
			& \text{Faster-RCNN} & N/A & 86.3 & \textbf{90.2} & \textbf{88.3} \\
			\hline
			\hline\rule{-2.2pt}{8pt}
			\multirow{16}*{WHEBD-759-SIM} & \multirow{8}*{\text{B-CNN}} & N/A & 63.7 & 64.6 & 64.2 \\
			\cline{3-6}\rule{-2.2pt}{8.5pt}\rule{-2.2pt}{8pt}
			& & SimCLR & 64.5 & 67.0 & 65.8 \\\rule{-2.2pt}{8pt}
			& & SwAV & 64.1 & 65.3 & 64.7 \\\rule{-2.2pt}{8pt}
			& & CMC & 66.5 & 68.7 & 67.6 \\
			\cline{3-6}\rule{-2.2pt}{8.5pt}
			& & \textbf{GloCo+SwAV} & 64.7 & 66.2 & 65.5 \\\rule{-2.2pt}{8pt}
			& & \textbf{GloCo+CMC} & 67.4 & 68.5 & 68.0 \\\rule{-2.2pt}{8pt}
			& & \textbf{LoCo} & 65.5 & \textbf{69.1} & 67.3 \\\rule{-2.2pt}{9pt}
			& & \textbf{CoCoNet} & \textbf{69.2} & 69.0 & \textbf{69.1} \\
			\cline{2-6}\rule{-2.2pt}{8pt}
			\vspace{-0.30cm}
			& \multicolumn{5}{c}{ } \\
			\cline{2-6}\rule{-2.2pt}{8.5pt}
			& \multirow{8}*{\text{Faster-RCNN}} & N/A & 61.3 & 66.7 & 64.0 \\
			\cline{3-6}\rule{-2.2pt}{8.5pt}\rule{-2.2pt}{8pt}
			& & SimCLR & 61.8 & 67.4 & 64.6 \\\rule{-2.2pt}{8pt}
			& & SwAV & 63.3 & 66.8 & 65.1 \\\rule{-2.2pt}{8pt}
			& & CMC & 65.0 & 68.5 & 66.8 \\ 
			\cline{3-6}\rule{-2.2pt}{8.5pt}
			& & \textbf{GloCo+SwAV} & 63.6 & 67.0 & 65.3 \\
			& & \textbf{GloCo+CMC} & \textbf{65.9} & 69.6 & 67.8 \\\rule{-2.2pt}{8pt}
			& & \textbf{LoCo} & 64.3 & 68.4 & 66.4 \\\rule{-2.2pt}{9pt}
			& & \textbf{CoCoNet} & 65.2 & \textbf{70.9} & \textbf{68.1} \\
			\hline
		\end{tabular}
	\end{center}
	\vskip -0.2in
\end{table*}

\subsubsection{Comparisons on benchmark graph prediction}
To evaluate the generalization of our proposed CoCoNet, we conducted comparisons in the field of graph prediction. As shown in Table \ref{tab:graphprediction}, our methods beat the compared baselines on most benchmark downstream tasks. However, comparing the results of the combination variants of GloCo, e.g., GloCo + GraphCL, with the original baselines, e.g., GraphCL, we observed that the improvement is limited (or even arbitrary). We attributed such a phenomenon to the learning paradigm of our proposed GloCo, which prompts the model to capture \textit{view-consistency} information by globally aligning the distributions of views. Such a process requires the data of different views to form \textit{different} and \textit{constant} distributions. Yet, we constructed views of graphs by using the \textit{same} and \textit{random} graph augmentations, which is contrary to our original intention of developing GloCo. So in the setting of conventional graph contrastive learning, GloCo's poor improvement is understandable. The experimental results further support the effectiveness of CoCoNet and the ablation variant LoCo, which proves that our proposed \textit{view-complementarity} information also applies to the graph self-supervised representation learning task, and mining such information can improve the benchmark methods. However, the effect of mining view-complementarity's boost on the model is reduced, because randomly and inconsistently generated multiple views degenerate the performance of our proposed CoCoNet.

\subsubsection{Comparisons on benchmark action recognition}
To evaluate the effectiveness of CoCoNet on data of another modality, we conducted comparisons in the field of action recognition, which is based on video data. The results, reported in Table \ref{tab:action}, support that our proposed methods can improve the performance of benchmark methods on the action recognition task of video data. Comparing the test results on UCF-101, we observed that CoCoNet and variants have remarkable \textit{task} transferability, and comparing the test results on HMDB-51, we found that our methods have the good \textit{task} and \textit{dataset} transferability. Therefore, on the action recognition task of video data, our proposed view-consistency and -complementary information is valuable to mine in the paradigm of self-supervised multiview video representation learning, and CoCoNet can effectively model such information.

\subsubsection{Comparisons on practical object detection}
In order to deal with real-world issues, we further performed the proposed CoCoNet on a practical dataset against state-of-the-art methods. As demonstrated in Table \ref{tab:whebd}, the comparison results on the practical dataset show that the methods in bold generally have better performance than other compared methods in the same experimental settings. The first 2 rows in the table represent the results of the pure supervised methods trained on the completely labeled WHEBD-759 dataset, and we observed that the models under such training strategy have the best performance, e.g., B-CNN trained on WHEBD-759 beats B-CNN w/ CoCoNet trained on WHEBD-759-SIM by 17.9\%, and Faster-RCNN trained on WHEBD-759 beats Faster-RCNN w/ CoCoNet trained on WHEBD-759-SIM by 20.2\% on average. This phenomenon indicates that the label information is important for models to learn discriminative representations, which is hard to be replaced by the self-supervised method. However, considering the fundamental idea that SSL can enhance the models to learn discriminative features, we used self-supervised methods to pretrain the encoders and then trained the supervised models on WHEBD-759-SIM to get better classification performance, and it is proved by the experiments. We first introduced the advanced self-supervised methods (e.g., SimCLR, SwAV, and CMC) to pretrain the encoders, and the results support that this pretraining procedure can improve the supervised models, where CMC has the best results, and, in detail, B-CNN w/ CMC beats the single B-CNN by 3.4\%, and Faster-RCNN w/ CMC beats Faster-RCNN by 2.8\% on average. Furthermore, the best results are always acquired by our proposed CoCoNet, for example, B-CNN w/ CoCoNet improves B-CNN w/ CMC by 1.5\%, and Faster-RCNN w/ CoCoNet improves Faster-RCNN w/ CMC by 1.3\% on average. The ablation models also have better performance than other self-supervised methods, which proves the effectiveness of the proposed method ulteriorly.

\subsection{Deepgoing exploration} \label{sec:deepgoing}
We conducted further experiments to explore the deep properties of CoCoNet.

\begin{table}[t]
	\vskip 0.05in
	\caption{Comparison of applying different settings of views.}
	\vskip -0.15in
	\label{tab:viewnumstudy}
	\setlength{\tabcolsep}{13pt}
	\begin{center}
		\begin{tabular}{c|c|c|l|c}
			\hline
			\multicolumn{3}{c|}{Views} & \multirow{2}*{Methods} & \multirow{2}*{Results} \\
			\cline{0-2}\rule{-2.2pt}{9pt}
			RGB & L & ab & & \\
			\hline
			\hline\rule{-2.2pt}{8pt}
			\multirow{4}*{$\checkmark$} & \multirow{4}*{$\checkmark$} & \multirow{4}*{} & SimCLR \cite{tc20} & 76.37 \\
			\rule{-2.2pt}{8pt}
			& & & SwAV \cite{cm20} & 66.14 \\\rule{-2.2pt}{8pt}
			& & & CMC \cite{ylt20} & 79.92 \\\rule{-2.2pt}{8pt}
			& & & \textbf{CoCoNet} & \textbf{80.88} \\
			\hline\rule{-2.2pt}{8pt}
			\multirow{4}*{$\checkmark$} & \multirow{4}*{} & \multirow{4}*{$\checkmark$} & SimCLR \cite{tc20} & 74.30 \\
			\rule{-2.2pt}{8pt}
			& & & SwAV \cite{cm20} & 64.72 \\\rule{-2.2pt}{8pt}
			& & & CMC \cite{ylt20} & 76.01 \\\rule{-2.2pt}{8pt}
			& & & \textbf{CoCoNet} & \textbf{76.35} \\
			\hline\rule{-2.2pt}{8pt}
			\multirow{4}*{} & \multirow{4}*{$\checkmark$} & \multirow{4}*{$\checkmark$} & SimCLR \cite{tc20} & 75.74 \\
			\rule{-2.2pt}{8pt}
			& & & SwAV \cite{cm20} & 65.90 \\\rule{-2.2pt}{8pt}
			& & & CMC \cite{ylt20} & 77.69 \\\rule{-2.2pt}{8pt}
			& & & \textbf{CoCoNet} & \textbf{78.26} \\
			\hline\rule{-2.2pt}{8pt}
			\multirow{4}*{$\checkmark$} & \multirow{4}*{$\checkmark$} & \multirow{4}*{$\checkmark$} & SimCLR \cite{tc20} & 80.58 \\
			\rule{-2.2pt}{8pt}
			& & & SwAV \cite{cm20} & 66.18 \\\rule{-2.2pt}{8pt}
			& & & CMC \cite{ylt20} & 81.31 \\\rule{-2.2pt}{8pt}
			& & & \textbf{CoCoNet} & \textbf{83.10} \\
			\hline
		\end{tabular}
	\end{center}
	\vskip -0.2in
\end{table}

\begin{figure*}
	\vskip 0in
	\begin{center}
		\centerline{\includegraphics[width=1.4\columnwidth]{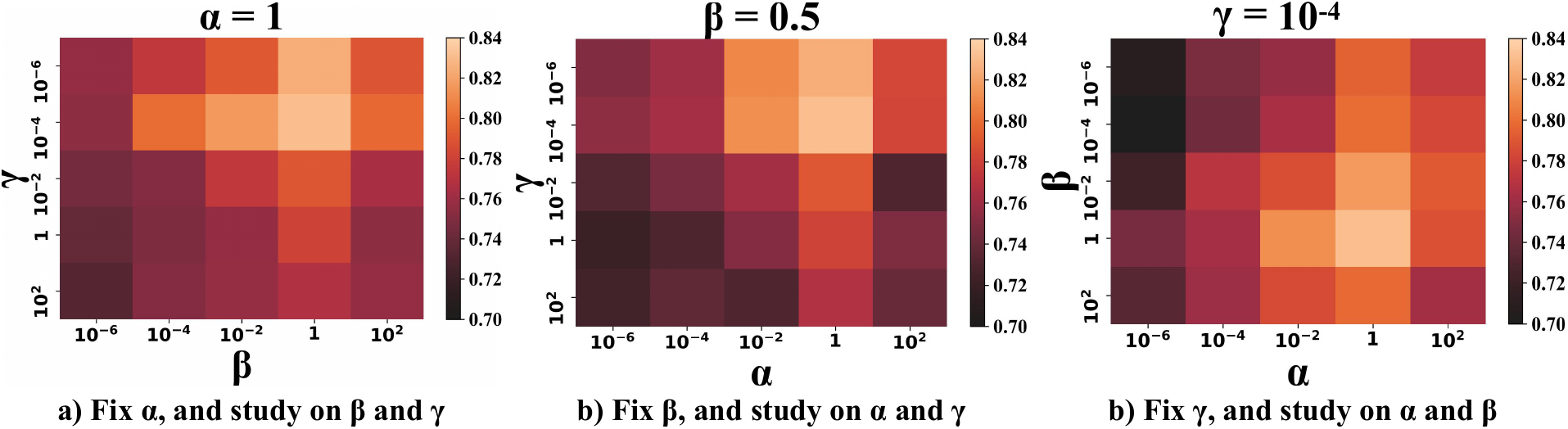}}
		\vskip -0.1in
		\caption{Influences of hyper-parameters $\alpha$, $\beta$ and $\gamma$ of CoCoNet. We conducted comparisons on the CIFAR10 dataset.}
		\label{fig:paramstudy}
	\end{center}
	\vskip -0.35in
\end{figure*}

\subsubsection{CoCoNet with different settings of views}
To validate whether CoCoNet has consistent performance under different settings of views, we selected several views from the RGB optical (RGB) view, the luminance (L) view, and the ab-color (ab) view and conducted experiments on CIFAR10 using conv encoder. As manifested in the Table \ref{tab:viewnumstudy}, CoCoNet has consistent performance and outperforms the compared methods on most tasks, and in details, CoCoNet beats the best benchmark method, i.e., CMC, by 0.96\% with RGB and L views, by 0.34\% with RGB and ab views, by 0.57\% with L and ab views, and by 1.79\% with all alternative views. We further added an experimental study on the comparison of the proposed CoCoNet and a typical contrastive learning method with more views. On the CIFAR10 dataset, we increased the number of views from 1 to 5 by sequentially adding the L, ab, RGB, Grayscale, and CbCr (belongs to YCbCr color space, where CB and Cr are the concentration offset components of blue and red) views. Results are shown in Figure \ref{fig:furtherviewsetting}. CoCoNet maintains its advantage over SimCLR when different view-settings are used for training. Compared with the addition of L, ab, and RGB, the additions of Grayscale and CbCr improve the performance of methods by a limited margin, and we considered the reason is that most information of Grayscale and CbCr is already contained by L, ab, and RGB. Concretely, our proposed consistency and complementarity regularization can indeed enhance the ability of the encoders to model multiple views, and such superiority is consistent under different view-settings.

\begin{figure}
    \begin{center}
        \begin{minipage}{0.265\textwidth}
            \includegraphics[width=\textwidth]{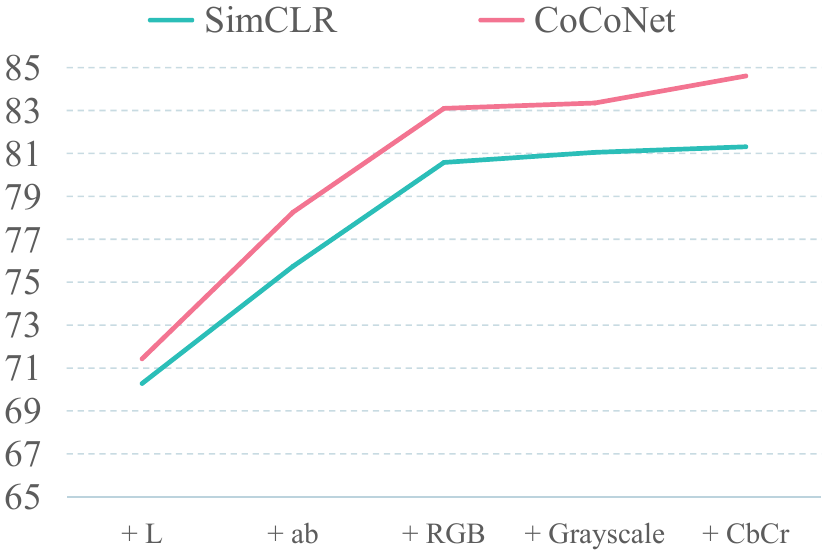}
        \end{minipage} \quad
        \begin{minipage}{0.20\textwidth}
            \caption{Comparisons with sequentially adding views on CIFAR10 using conv, which further indicates the superiority of CoCoNet over the compared baseline under different view-settings.}
            \label{fig:furtherviewsetting}
        \end{minipage}
    \end{center}
    \vspace{-0.6cm}
\end{figure}

\subsubsection{Hyper-parameter heatmap}
Specifically, we performed several experiments to study the influence of the tunable hyper-parameters. The hyper-parameter $\alpha$ balances the impact of the local complementarity preserving module. $\beta$ balances the impact of conventional contrastive learning loss. $\gamma$ balances the impact of the global consistency preserving module. To explore the influence of $\alpha$ and $\beta$, we fixed $\gamma$ and selected $\alpha$ from the range of \{$10^{-6}, 10^{-4}, 10^{-2}, 1, 10^{2}$\} and $\beta$ from the range of \{$10^{-6}, 10^{-4}, 10^{-2}, 1, 10^{2}$\}. Following the same principle, we selected $\gamma$ from the range of \{$10^{-6}, 10^{-4}, 10^{-2}, 1, 10^{2}$\}. As $a)$, $b)$, and $c)$ shown in Figure \ref{fig:paramstudy}, we observed that good classification performance is highly dependent on the local complementarity preserving module, i.e., $\alpha$. An appropriate tuning of the impact of the contrastive loss, i.e., $\beta$, is needed for CoCoNet to enhance the cross-view feature discriminability. As such, the global consistency preserving module helps in classification performance with a small amount of $\gamma$, because it aligns the distribution of views, which helps to model the view-shared information.

\begin{figure*}
	\vskip 0in
	\begin{center}
		\centerline{\includegraphics[width=1.8\columnwidth]{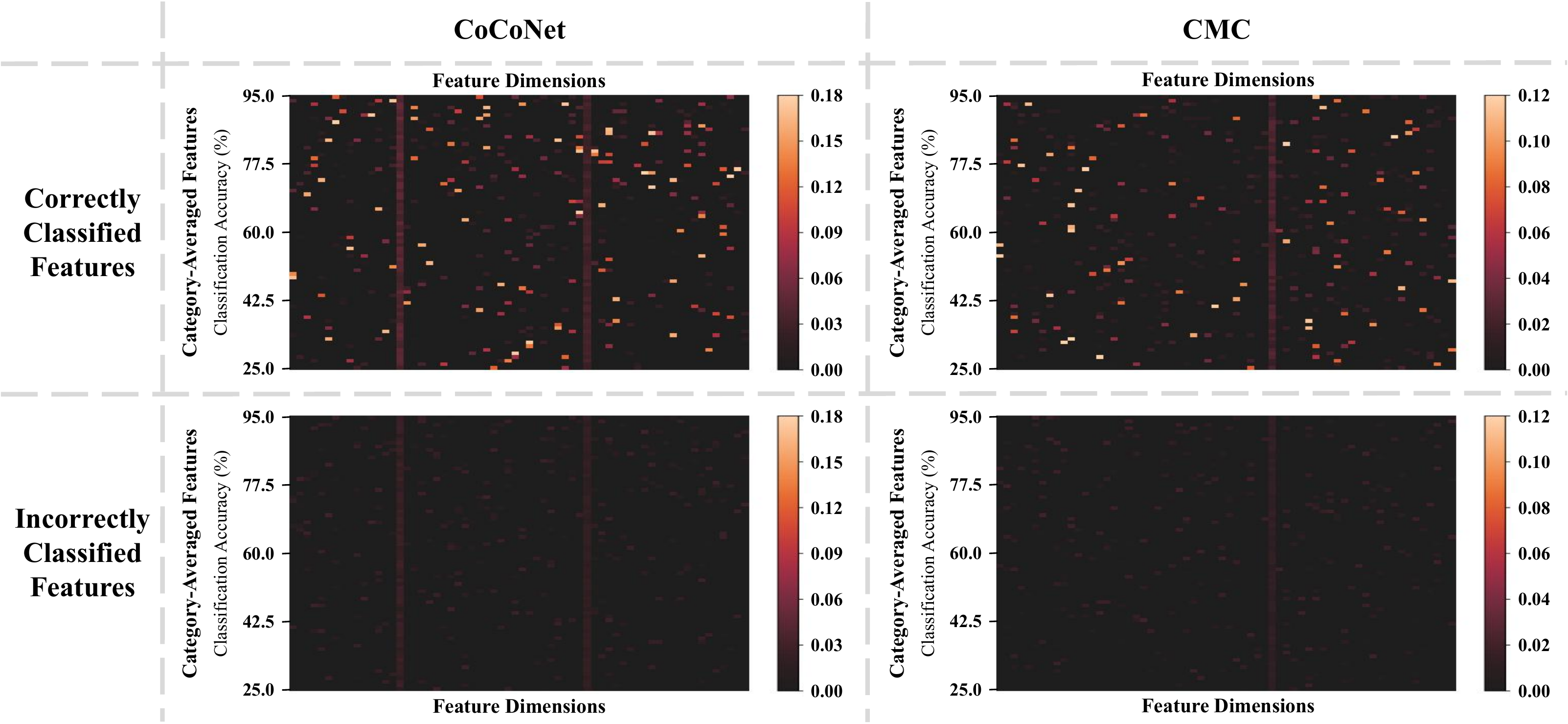}}
		\vskip -0.1in
		\caption{Visual comparisons of the top category-averages features of correct and incorrect classifications in the representation space of CoCoNet and CMC. We derived averaged features for each category, and according to the classification results, we retrieved the top category-averages features to evaluate the \textit{activated} feature elements of CoCoNet and CMC, which is conducted on the Tiny ImageNet dataset by following the experimental principle of \cite{DBLP:journals/corr/abs-2203-01881}. We observe that the correct classification contains specific activated feature elements that are more salient (colorful) than other feature elements, whereas the incorrect classifications do not.}
		\label{fig:casestudy}
	\end{center}
	\vskip -0.35in
\end{figure*}

\subsubsection{CoCoNet with different discrepancy metrics} \label{sec:diffdis}
We conducted an ablation comparison by employing different discrepancy metrics for the proposed method. As shown in Table \ref{tab:diffdis}, we directly replaced the discrepancy metric in GloCo module with KL, WD, etc. We observed that no matter which discrepancy metric is based on, GloCo + CMC can improve CMC, which proves the effectiveness of aligning the distributions of multiple views. Yet the improvements in taking different discrepancy metrics are inconsistent. Generally, the Wasserstein distance-based methods beat the KL divergence-based method, and we discussed the reasons in Section \ref{sec:advgswd}. Since directly calculating the high-dimensional Wasserstein distance is extremely computationally expensive, the difference between the Wasserstein distance-based methods is that the approaches to approximately calculate Wasserstein distances. In detail, WD uses the dual form of Wasserstein distance, yet the Lipschitz constraint is difficult to meet. SWD first obtains the one-dimensional representation of the high-dimensional probability distribution through linear mapping and then calculates the Wasserstein distance of the one-dimensional representation of the two probability distributions. Likewise, GSWD uses a similar approach except that generalized nonlinear mapping is used instead of linear mapping. The results demonstrate that, in the setting of multi-view learning, GSWD can retain more discriminative information than SWD in dimensionality reduction.

\begin{table}[t]
	\renewcommand\arraystretch{1.1}
	\vskip 0.05in
	\caption{Classification top-1 accuracy (\%) on the CIFAR10 and Tiny ImageNet datasets. We conducted several experiments based on the \textit{conv} encoder and classifier as in Table \ref{tab:a}. We introduced the optional discrepancy metrics, e.g., KL-divergence (KL) \cite{2003Goldbberger}, Wasserstein distance (WD) \cite{kuroki2019}, sliced Wasserstein distance (SWD) \cite{lee2019sliced}, and generalized sliced Wasserstein distance (GSWD) \cite{DBLP:conf/nips/KolouriNSBR19}, to GloCo + CMC. Notably, the GSWD-based GloCo + CMC outperforms benchmark methods but falls short compared to CoCoNet.}
	\vskip -0.15in
	\label{tab:diffdis}
	\setlength{\tabcolsep}{4.pt}
	\begin{center}
		\begin{tabular}{l|ccc}
			\hline
			\text{Model} & CIFAR10 & Tiny ImageNet & Average\\
			\hline
			\text{CMC} & 81.31 & 41.58 & 61.45 \\
			\hline
			\text{GloCo + CMC w/ KL} & 81.62 & 42.08 & 61.85 \\
			\text{GloCo + CMC w/ WD} & 82.02 & 42.14 & 62.08 \\
			\text{GloCo + CMC w/ SWD} & 82.07 & 42.05 & 62.06 \\
			\textbf{GloCo + CMC w/ GSWD} & 82.27 & 42.21 & 62.24 \\
			\hline
			\textbf{CoCoNet} & \textbf{83.10} & \textbf{42.28} & \textbf{62.69} \\
			\hline
		\end{tabular}
	\end{center}
	\vspace{-0.5cm}
\end{table}

\begin{figure}
    \begin{center}
        \begin{minipage}{0.23\textwidth}
            \includegraphics[width=\textwidth]{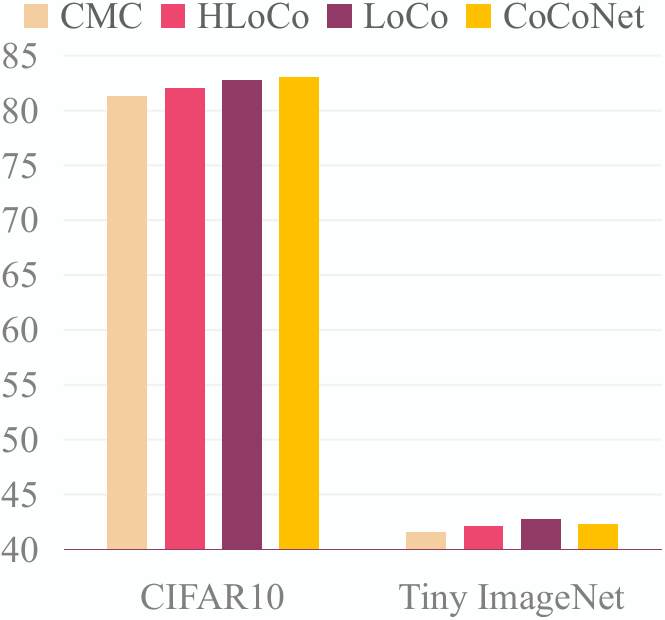}
        \end{minipage} \quad
        \begin{minipage}{0.21\textwidth}
            \caption{Research on the effectiveness of a specific sub-structure of the LoCo module. Specifically, we decoupled modeling the low-level feature information from learning the complementarity-factor $CF$ and evaluate the ablation model.}
            \label{fig:lowlevelabl}
        \end{minipage}
    \end{center}
    \vspace{-0.6cm}
\end{figure}

\subsubsection{Validating the effectiveness of modeling low-level information of the LoCo module} \label{sec:lowlevelabl}
To decouple the design of the architecture and the design of the learning objective, we conducted a further exploration with CoCoNet, LoCo, and an ablation model HLoCo by removing the low-level feature maps from LoCo, i.e., HLoCo only uses the high-level feature vectors, not the information of low-level feature maps. As shown in Figure \ref{fig:lowlevelabl} and Table \ref{tab:a}, we observed that HLoCo beats the baselines with the same high-level representations, which shows the effectiveness of the learning objective $\mathcal{L}_{LoCo}$. Moreover, both LoCo and CoCoNet can outperform HLoCo on benchmark datasets, indicating that modeling the discriminative information from low-level feature maps can generally improve the performance of our method. The reason behind such a phenomenon is that from the perspective of the information theory, compared with the low-level feature map, the high-level feature vector may loss some complementarity information. Therefore, according to the amount of information entropy, HLoCo, like typical contrastive learning methods \cite{tc20, cm20, ylt20}, only contains the information of high-dimensional feature vectors, while the useful complementarity information may be lost within the encoding process so that compared with LoCo and CoCoNet, complementarity information is not sufficiently explored by HLoCo. We further observed that comparing HLoCo and LoCo, LoCo improves HLoCo by a larger margin on Tiny ImageNet than on CIFAR10. We reckoned that the classification on Tiny ImageNet requires more complementarity information, since Tiny ImageNet contains 200 categories while CIFAR10 only contains 10 categories.

\begin{figure}
    \begin{center}
        \begin{minipage}{0.23\textwidth}
            \includegraphics[width=\textwidth]{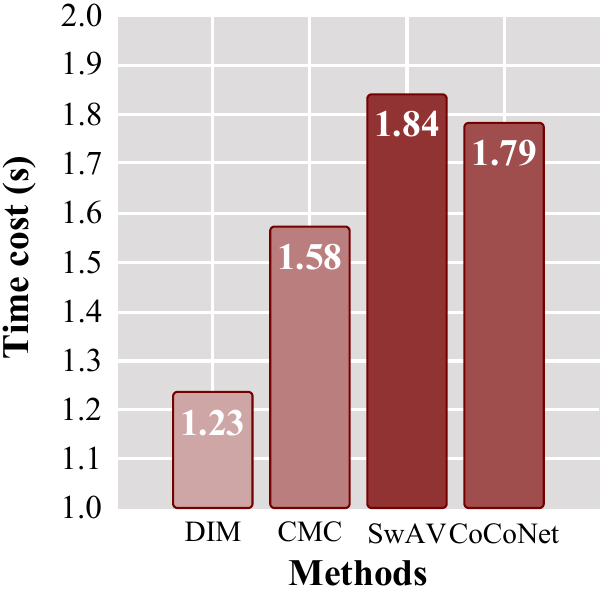}
        \end{minipage} \quad
        \begin{minipage}{0.21\textwidth}
            \caption{The average computational time costs of the training of a batch during the first 20 epochs. The process includes the feed-forward calculation and the back-propagation training of the encoders.}
            \label{fig:tcchart}
        \end{minipage}
    \end{center}
    \vspace{-0.6cm}
\end{figure}

\begin{figure*}
	\vskip 0in
	\begin{center}
		\centerline{\includegraphics[width=1.55\columnwidth]{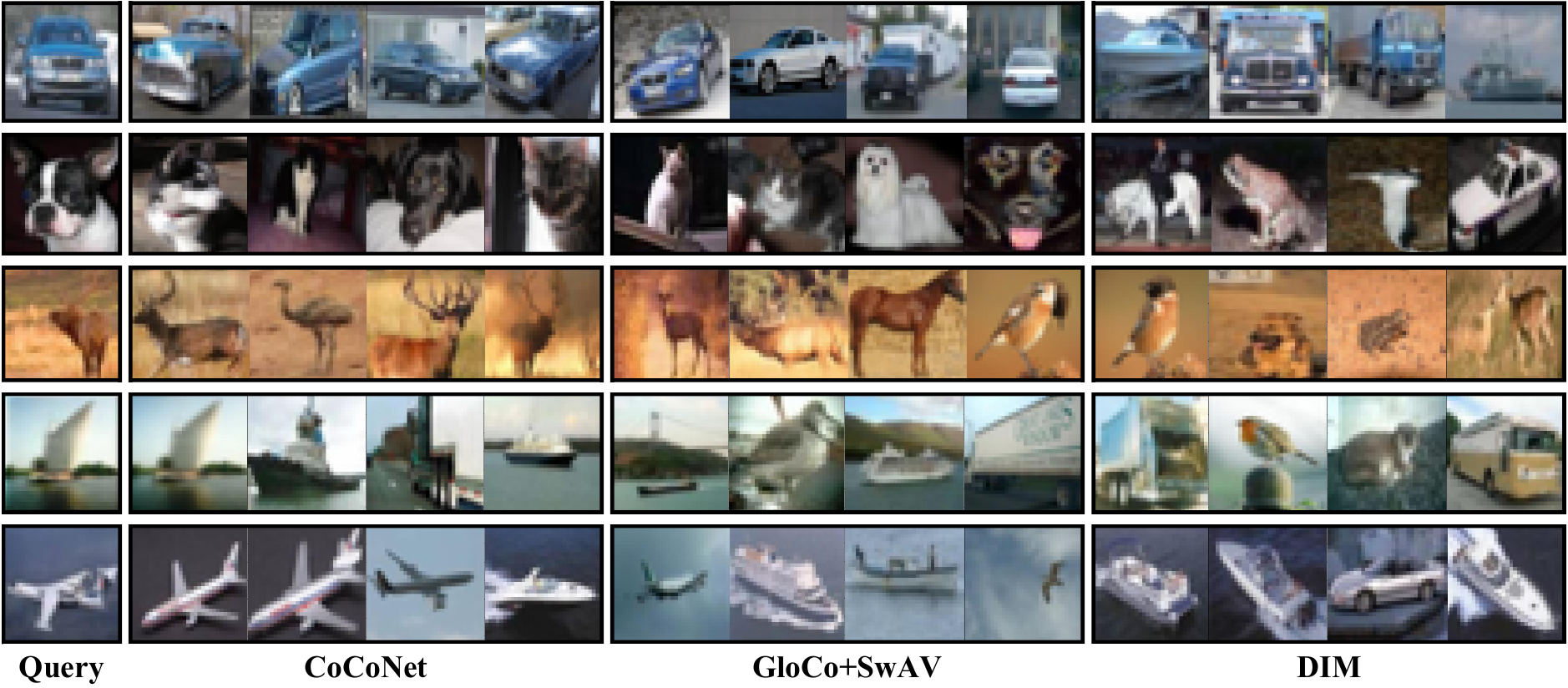}}
		\vskip -0.1in
		\caption{Visual comparisons for studying the merits of CoCoNet on the CIFAR10 dataset. We retrieved the 4 nearest neighbors to evaluate the discriminability by using $L_1$ distance. The leftmost images are randomly selected images as queries, and the other images are their nearest neighbors measured in the representations of compared methods.}
		\label{fig:queryplots}
	\end{center}
	\vskip -0.35in
\end{figure*}

\subsubsection{Case study} \label{sec:casestudy}
As shown in Figure \ref{fig:casestudy}, each class has a discriminatory set of feature elements that contribute to correct classifications, i.e, each category requires different sets of discriminative feature elements for classifications. The observations on the incorrect classifications demonstrate that the over-consistency (trivial) of feature elements is a crucial reason that the feature cannot be correctly classified, indicating that the misclassified features model much task-irrelevant information. We observe that compared with CMC, CoCoNet learns features with more salient elements, indicating that CoCoNet can learn more activated feature elements for each category in classifications. The reason behind such an observation is that CMC only leverages the typical contrastive approach to model consistency information, while CoCoNet further imposes the proposed LoCo module to extract complementarity information. Hence, in addition to the feature elements modeling consistency information, the feature elements modeling complementary information can also contribute to the classification of each category. The extra elements in the features learned by our method can be regarded as the elements modeling complementary information. Therefore, a larger amount of discriminative information can empower CoCoNet to be robust to task-irrelevant noisy information, resulting in better performance on downstream tasks.

\subsubsection{Limitations and discussion} \label{sec:limitations}
\textbf{Discussion on the time complexity.} In head-to-head comparisons, CoCoNet achieves the state-of-the-art, which supports that mining view-consistency and -complementarity knowledge can improve to model multiple views in multi-view SSL. Yet compared with benchmark self-supervised methods, CoCoNet has relatively higher time complexity in training. As shown in Figure \ref{fig:tcchart}, the computational time cost of CoCoNet is lower than SwAV but higher than the baseline CMC. We reckoned the reasons are 1) the training of the critic network of GloCo; 2) the matrix operations of LoCo. However in the test, the compared methods adopt the same paradigm, and the test time complexities are the same.

\textbf{Threats to validity \cite{DBLP:books/daglib/0029933}.} For the \textit{conclusion validity}, we followed the benchmark experimental settings \cite{rdh19, ylt20}, e.g., choice of statistical tests, choice of sample size, etc. In order to avoid the threat to validity caused by imbalanced datasets, in addition to accuracy, we further adopted F1-Measure as a metric to measure the experiments, which is shown in Table \ref{tab:f1measure}. For the \textit{internal validity}, we introduced sufficient ablation studies, demonstrated in Section \ref{sec:results}, to prove the effectiveness of the proposed parts of CoCoNet, i.e., GloCo and LoCo. To further explore whether replacing specific components of CoCoNet with variants may affect the conclusion that \textit{``the improvement in results is due to the proposed method''}, we conducted comparisons in Table \ref{tab:diffdis} and Figure \ref{fig:lowlevelabl}, and the results support the effectiveness of CoCoNet's components. For the \textit{construct validity}, a foundational assumption of multi-view SSL is stated in Assumption \ref{ass:1}, which is theoretically proved by \cite{wang2022chaos, 2008Sridharan, 2013Xu}. Moreover, such an assumption is empirically proved by \cite{ylt20, tc20, cm20} to be applicable to image-related tasks, by \cite{you2020graph} to be applicable to graph-related tasks, and by \cite{ylt20} to be applicable to video-related tasks. For the \textit{external validity}, to avoid the influence of random factors (such as random seeds) in the experiment on the results, we collected the results of 5 trials for comparisons. The average result of the last 10 epochs is used as the final result of each trial. The average results from all trials are presented in tables. We conducted comparisons on multiple downstream tasks, including image classification tasks, graph prediction tasks, and action recognition tasks, to avoid artificial experimental settings that may affect the generalization of the model. In order to further verify whether the experiments on benchmark datasets can be generalized to actual real-world scenarios, we conducted comparisons on a practical dataset, i.e., WHEBD-759, and the results demonstrate that CoCoNet can still improve the performance of benchmark supervised methods in a self-supervised manner.

\subsubsection{Visual comparisons}
As shown in Figure \ref{fig:queryplots}, the representations learned by CoCoNet lead to more interpretable metric structures since neighboring representations correspond to visually similar images of the same category. There are three reasons for this circumstance: 1) CoCoNet learns representations from multiple views instead of a single view; 2) LoCo helps to refine the representations by improving the feature's view-specific discriminability; 3) GloCo further enhances the learned representations' view-shared discriminability by measuring the discrepancy metric between views.

\section{Conclusions}
This paper proposes a novel CoCoNet to mine discriminative knowledge from multi-view data in an unsupervised manner. To this end, CoCoNet globally aligns the distributions of views in the latent space by adopting an efficient alignment method based on GSWD, which helps to capture view-consistency information. CoCoNet leverages the proposed complementarity-factor to maintain the cross-view complementarity of the latent representations on the local stage. Compared with the conventional methods, CoCoNet explores more, albeit still not full, discriminative information from multiple views. The provided theoretical and experimental analyses support the effectiveness of CoCoNet.


\section{Acknowledgements}
The authors would like to thank the associate editor and anonymous reviewers for their valuable comments. This work is supported in part by the Strategic Priority Research Program of the Chinese Academy of Sciences, Grant No. XDA19020500, National Natural Science Foundation of China No. 61976206 and No. 61832017, Key Special Project for Introduced Talents Team of Southern Marine Science and Engineering Guangdong Laboratory (Guangzhou), No. GML2019ZD0603, Beijing Outstanding Young Scientist Program NO. BJJWZYJH012019100020098, Beijing Academy of Artificial Intelligence (BAAI), China Unicom Innovation Ecological Cooperation Plan, the Fundamental Research Funds for the Central Universities, the Research Funds of Renmin University of China 21XNLG05, and Public Computing Cloud, Renmin University of China. This work is also supported in part by Intelligent Social Governance Platform, Major Innovation \& Planning Interdisciplinary Platform for the ``Double-First Class'' Initiative, Renmin University of China, and Public Policy and Decision-making Research Lab of Renmin University of China.


\ifCLASSOPTIONcaptionsoff
  \newpage
\fi



\bibliographystyle{IEEEtran}
\bibliography{reference}
%



%
\vskip -0.5in
\begin{IEEEbiography}[{\includegraphics[width=1in,height=1.25in,clip,keepaspectratio]{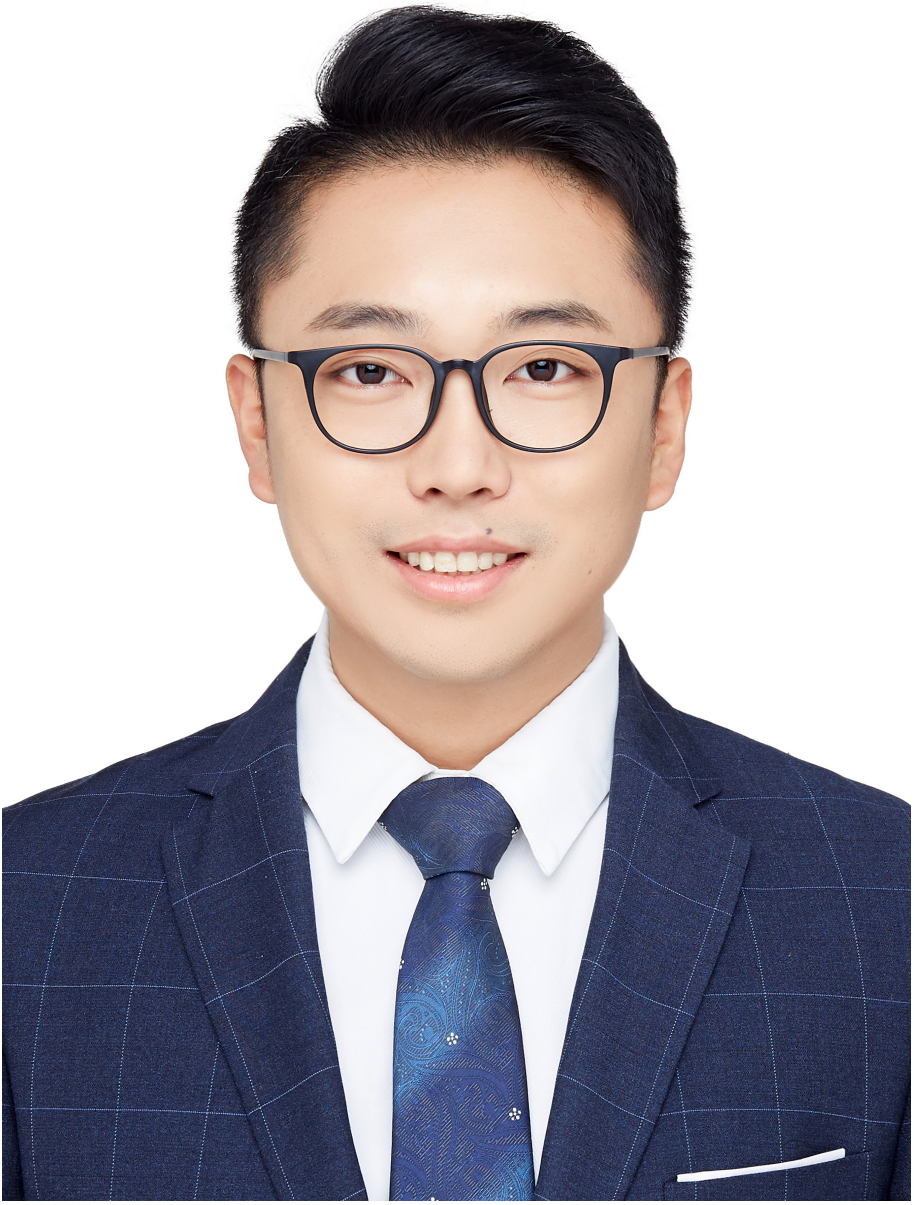}}]{Jiangmeng Li}
	received the BS degree in the department of software engineering, Xiamen University, Xiamen, China, in 2016, and the MS degree from New York University, New York, USA, in 2018. He is currently a doctoral student at the University of Chinese Academy of Sciences. His research interests include self-supervised learning, deep learning, and machine learning. He has published more than five papers in journals and conferences such as IEEE Transactions on Knowledge and Data Engineering (TKDE), International Conference on Machine Learning (ICML), International Joint Conference on Artificial Intelligence (IJCAI), etc.
\end{IEEEbiography}
\vskip -0.5in
\begin{IEEEbiography}[{\includegraphics[width=1in,height=1.25in,clip,keepaspectratio]{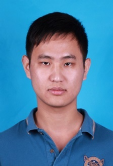}}]{Wenwen Qiang}
	received the MS degree in the department of mathematics, college of science, China Agricultural University, Beijing, in 2018. He is currently a doctoral student at the University of Chinese Academy of Sciences. His research interests include transfer learning, deep learning, and machine learning. He has published more than five papers in journals and conferences such as IEEE Transactions on Knowledge and Data Engineering (TKDE), International Conference on Machine Learning (ICML), International Joint Conference on Artificial Intelligence (IJCAI), etc.
\end{IEEEbiography}
\vskip -0.5in


\begin{IEEEbiography}[{\includegraphics[width=1in,height=1.25in,clip,keepaspectratio]{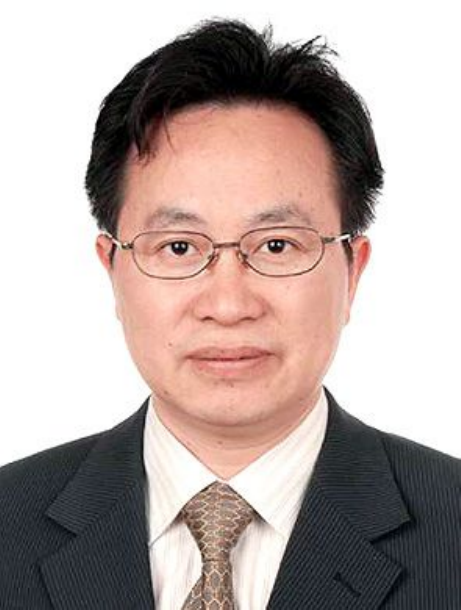}}]{Changwen Zhen}
	received the Ph.D. degree in Huazhong University of Science and Technology. He is currently a professor in Institute of Software, Chinese Academy of Science. His research interests include computer graph and artificial intelligence.
\end{IEEEbiography}
\vskip -0.5in

\begin{IEEEbiography}[{\includegraphics[width=1in,height=1.25in,clip,keepaspectratio]{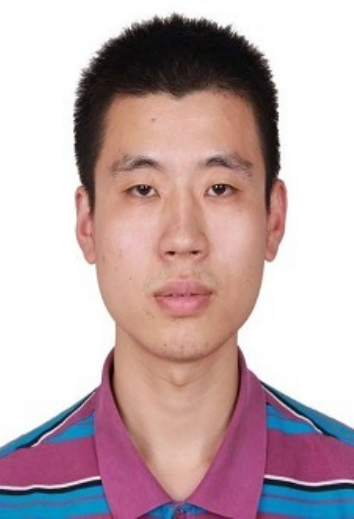}}]{Bing Su}
received the BS degree in information engineering from the Beijing Institute of Technology, Beijing, China, in 2010, and the PhD degree in electronic engineering from Tsinghua University, Beijing, China, in 2016. From 2016 to 2020, he worked with the Institute of Software, Chinese Academy of Sciences, Beijing. Currently, he is an associate professor with the Gaoling School of Artificial Intelligence, Renmin University of China. His research interests include pattern recognition, computer vision, and machine learning. He has published more than ten papers in journals and conferences such as IEEE Transactions on Pattern Analysis and Machine Intelligence (TPAMI), IEEE Transactions on Image Processing (TIP), International Conference on Machine Learning (ICML), IEEE Conference on Computer Vision and Pattern Recognition (CVPR), IEEE International Conference on Computer Vision (ICCV), etc.
\end{IEEEbiography}
\vskip -0.5in

\begin{IEEEbiography}[{\includegraphics[width=1in,height=1.25in,clip,keepaspectratio]{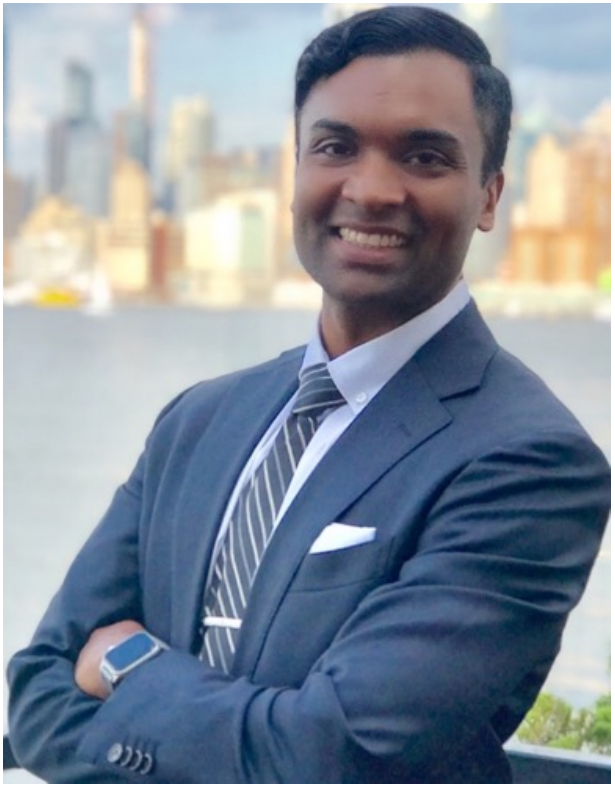}}]{Farid Razzak}
	received the Ph.D. degree from the Management Science \& Information Systems Department at Rutgers, the State University of New Jersey in 2020, the M.S. degree from the School of Professional Studies at New York University in 2009, and received a B.B.A degree from Bernard M. Baruch College at City University of New York in 2007.  He is currently an Financial Quantitative Data Scientist for the Securities and Exchange Commission as well as adjunct faculty at New York University \& Columbia University. His research interests include applied data mining for financial regulations, business applications and services.
\end{IEEEbiography}
\vskip -0.5in

\begin{IEEEbiography}[{\includegraphics[width=1in,height=1.25in,clip,keepaspectratio]{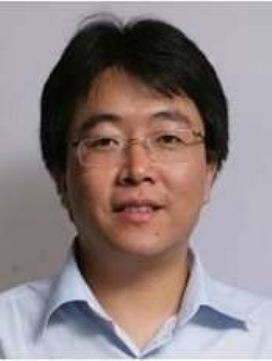}}]{Ji-Rong Wen}
	is a full professor at Gaoling School of Artificial Intelligence, Renmin University of China. He worked at Microsoft Research Asia for fourteen years and many of his research results have been integrated into important Microsoft products (e.g. Bing). He serves as an associate editor of ACM Transactions on Information Systems (TOIS). He is a Program Chair of SIGIR 2020. His main research interests include web data management, information retrieval, data mining and machine learning.
\end{IEEEbiography}
\vskip -0.5in

\begin{IEEEbiography}[{\includegraphics[width=1in,height=1.25in,clip,keepaspectratio]{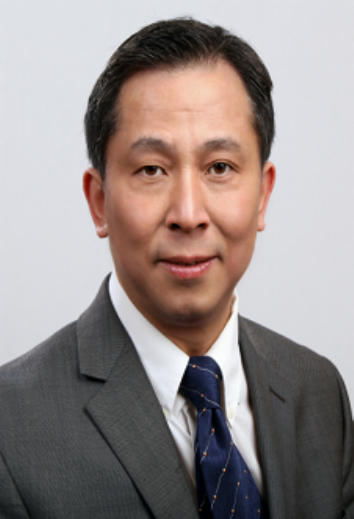}}]{Hui Xiong}
received his Ph.D. in Computer Science from the University of Minnesota - Twin Cities, USA, in 2005, the B.E. degree in Automation from the University of Science and Technology of China (USTC), Hefei, China, and the M.S. degree in Computer Science from the National University of Singapore (NUS), Singapore. He is a chair professor at the Hong Kong University of Science and Technology (Guangzhou). He is also a Distinguished Professor at Rutgers, the State University of New Jersey, where he received the 2018 Ram Charan Management Practice Award as the Grand Prix winner from the Harvard Business Review, RBS Dean's Research Professorship (2016), two-year early promotion/tenure (2009), the Rutgers University Board of Trustees Research Fellowship for Scholarly Excellence (2009), the ICDM-2011 Best Research Paper Award (2011), the Junior Faculty Teaching Excellence Award (2007), Dean's Award for Meritorious Research (2010, 2011, 2013, 2015) at Rutgers Business School, the 2017 IEEE ICDM Outstanding Service Award (2017), and the AAAI-2021 Best Paper Award (2021). Dr. Xiong is also a Distinguished Guest Professor (Grand Master Chair Professor) at the University of Science and Technology of China (USTC). For his outstanding contributions to data mining and mobile computing, he was elected an ACM Distinguished Scientist in 2014, an IEEE Fellow and an AAAS Fellow in 2020. His general area of research is data and knowledge engineering, with a focus on developing effective and efficient data analysis techniques for emerging data intensive applications. He currently serves as a co-Editor-in-Chief of Encyclopedia of GIS (Springer) and an Associate Editor of IEEE Transactions on Data and Knowledge Engineering (TKDE), IEEE Transactions on Big Data (TBD), ACM Transactions on Knowledge Discovery from Data (TKDD) and ACM Transactions on Management Information Systems (TMIS). He has served regularly on the organization and program committees of numerous conferences, including as a Program Co-Chair of the Industrial and Government Track for the 18th ACM SIGKDD International Conference on Knowledge Discovery and Data Mining (KDD), a Program Co-Chair for the IEEE 2013 International Conference on Data Mining (ICDM), a General Co-Chair for the IEEE 2015 International Conference on Data Mining (ICDM), and a Program Co-Chair of the Research Track for the 24th ACM SIGKDD International Conference on Knowledge Discovery and Data Mining (KDD2018).

\end{IEEEbiography}




\clearpage

\begin{figure*}
	\vskip 0in
	\begin{center}
		\centerline{\includegraphics[width=1.6\columnwidth]{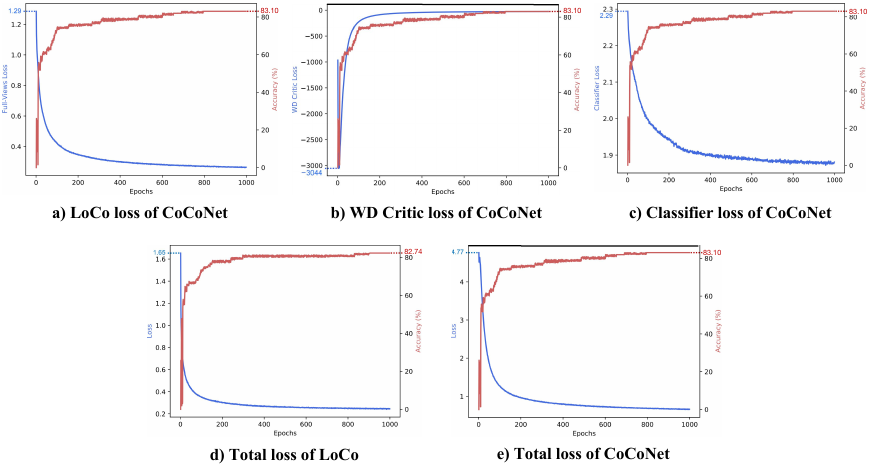}}
		\vskip -0.1in
		\caption{Extended verification for studying the loss convergence properties of CoCoNet and LoCo on the CIFAR10 dataset.}
		\label{fig:lossconverge}
	\end{center}
	\vskip -0.35in
\end{figure*}

\section{Appendix}
\subsection{Theoretical proofs} \label{sec:proof}
In Section \ref{sec:ta}, we propose two theorems: Theorem \ref{the:1}, i.e., the View-Consistency information with a potential loss of View-Specific Noise $\epsilon^{noise}$ theorem; Theorem \ref{the:2}, i.e., the View-Complementarity information, which is view-specific and task-relevant theorem. Here, we provide a formalized view to describe the proofs of them.

\subsubsection{Proof of Theorem \ref{the:1}} \label{pro:1}
We validate that $I(X^1;Y) \geq I(X^1;Y|\epsilon^{noise})$ by introducing the KL-divergence \cite{1951Ls} measurement into the calculation of mutual information:

\begin{proof}
	To proof $I(X^1;Y) \geq I(X^1;Y|\epsilon^{noise})$\\
	\\$\because I(X;Y) = \sum\limits_{x \in X}\sum\limits_{y \in Y}{\mathcal{P}(x, y)}{\log\frac{\mathcal{P}(x, y)}{{\mathcal{P}(x)\cdot}\mathcal{P}(y)}}$\\
	\\$\therefore I(X^1;Y) = \sum\limits_{x \in X^1}\sum\limits_{y \in Y}{\mathcal{P}(x, y)}{\log\frac{\mathcal{P}(x, y)}{{\mathcal{P}(x)\cdot}\mathcal{P}(y)}}$\\
	\\$\therefore I(X^1;Y|\epsilon^{noise}) = \sum\limits_{x \in X^1}\sum\limits_{y \in \{Y - \epsilon^{noise}\}}{\mathcal{P}(x, y)}{\log\frac{\mathcal{P}(x, y)}{{\mathcal{P}(x)\cdot}\mathcal{P}(y)}}$\\
	\\And KL-divergence is defined as:\\
	\\$D_{KL}(P||Q) = \int\limits{\mathcal{P}(x)\log\frac{\mathcal{P}(x)}{\mathcal{Q}(x)}}dx$\\
	\\The discrete form of KL-divergence is:\\
	\\$D_{KL}(P||Q) = \sum\limits{\mathcal{P}(x)\log\frac{\mathcal{P}(x)}{\mathcal{Q}(x)}}$\\
	\\We try to use KL divergence to fit the calculation of mutual information, and the $\mathcal{P}$ and $\mathcal{Q}$ are approximated by:\\
	\\$\hat{\mathcal{P}}(x) = \mathcal{P}(x, y)$\\
	\\$\hat{\mathcal{Q}}(x) = \mathcal{P}(x) \cdot \mathcal{P}(y)$\\
	\\Put $\hat{\mathcal{P}}(x)$ and $\hat{\mathcal{Q}}{(x)}$ into the above formula of the discrete KL-divergence:\\
	\\$D_{KL}(\mathcal{P}_{XY}||\mathcal{P}_X\mathcal{P}_Y) = \sum\limits_{x \in X}\sum\limits_{y \in Y}{\mathcal{P}(x, y)}{\log\frac{\mathcal{P}(x, y)}{{\mathcal{P}(x)\cdot}\mathcal{P}(y)}}$\\
	\\Then, we get:\\
	\\$D_{KL}(\mathcal{P}_{XY}||\mathcal{P}_X\mathcal{P}_Y) = I(X;Y)$\\
	\\$\therefore I(X^1;Y) = D_{KL}(\mathcal{P}_{X^1Y}||\mathcal{P}_{X^1}\mathcal{P}_Y)$\\
	\\$\therefore I(X^1;Y|\epsilon^{noise}) = D_{KL}(\mathcal{P}_{X^1\{Y-\epsilon^{noise}\}}||\mathcal{P}_{X^1}\mathcal{P}_{\{Y-\epsilon^{noise}\}})$\\
	\\Because $Y$ is not fully compressed, which means $I(X^1;Y|T) \geq 0$, it is acknowledged that $\epsilon^{noise} \geq 0$. For the KL-divergence, $\mathcal{P}_{X^1}$ is constant, and $Y \geq \{Y-\epsilon^{noise}\}$. Therefore, compared with the joint $\mathcal{P}_{X^1Y}$ and $\mathcal{P}_{X^1} \cdot \mathcal{P}_Y$, the distributions of the joint $\mathcal{P}_{X^1\{Y-\epsilon^{noise}\}}$ and $\mathcal{P}_{X^1} \cdot \mathcal{P}_{\{Y-\epsilon^{noise}\}}$ are more consistent, and then we get:\\
	\\\begin{small}
		$D_{KL}(\mathcal{P}_{X^1\{Y-\epsilon^{noise}\}}||\mathcal{P}_{X^1}\mathcal{P}_{\{Y-\epsilon^{noise}\}}) \leq D_{KL}(\mathcal{P}_{X^1Y}||\mathcal{P}_{X^1}\mathcal{P}_Y)$
	\end{small}\\
	\\$\therefore I(X^1;Y) \geq I(X^1;Y|\epsilon^{noise})$
	\\
\end{proof}

\subsubsection{Proof of Theorem \ref{the:2}} \label{pro:2}
We validate that $I(Y;T) \leq I(Y;T) + I(X^1;Y^*;T|X^2) + I(X^2;Y^*;T|X^1)$ by introducing the KL-divergence \cite{1951Ls} measurement into the calculation of mutual information:

\begin{proof}
	To proof $I(Y;T) \leq I(Y;T) + I(X^1;Y^*;T|X^2) + I(X^2;Y^*;T|X^1)$\\
	\\Transpose the mentioned equation:\\
	\\$I(Y;T) - I(Y;T) \leq I(X^1;Y^*;T|X^2) + I(X^2;Y^*;T|X^1)$\\
	\\$I(X^1;Y^*;T|X^2) + I(X^2;Y^*;T|X^1) \geq 0$\\
	\\Since, we assume that $Y^*$ is a extended representation of $Y$, and it can contain part of the View-Complementarity information, i.e., $I(X^1;T|X^2) + I(X^2;T|X^1)$. Therefore, we only need to proof that $I(X^1;T|X^2)$ or $I(X^2;T|X^1)$ is not null, because the mutual information cannot be negative. The proof is reformed to:\\
	\\$I(X^1;T|X^2) \geq 0$\\
	\\$I(X^2;T|X^1) \geq 0$\\
	\\$\because I(X;Y) = \sum\limits_{x \in X}\sum\limits_{y \in Y}{\mathcal{P}(x, y)}{\log\frac{\mathcal{P}(x, y)}{{\mathcal{P}(x)\cdot}\mathcal{P}(y)}}$\\
	\\$\therefore I(X^1;T|X^2) = \sum\limits_{x \in X^1}\sum\limits_{y \in \{T-X^2\}}{\mathcal{P}(x, y)}{\log\frac{\mathcal{P}(x, y)}{{\mathcal{P}(x)\cdot}\mathcal{P}(y)}}$\\
	\\$\therefore I(X^2;T|X^1) = \sum\limits_{x \in X^2}\sum\limits_{y \in \{T-X^1\}}{\mathcal{P}(x, y)}{\log\frac{\mathcal{P}(x, y)}{{\mathcal{P}(x)\cdot}\mathcal{P}(y)}}$\\
	\\As the equation deducing in Proof \ref{pro:1}, we use the discrete form of KL divergence to fit the calculation of mutual information, and then we get:\\
	\\$D_{KL}(\mathcal{P}_{XY}||\mathcal{P}_X\mathcal{P}_Y) = I(X;Y)$\\
	\\$\therefore I(X^1;T|X^2) = D_{KL}(\mathcal{P}_{X^1{\{T-X^2\}}}||\mathcal{P}_{X^1}\mathcal{P}_{\{T-X^2\}})$\\
	\\$\therefore I(X^2;T|X^1) = D_{KL}(\mathcal{P}_{X^2{\{T-X^1\}}}||\mathcal{P}_{X^2}\mathcal{P}_{\{T-X^1\}})$\\
	\\The downstream task-relevant information $T$ is not fully contained in any view of data, e.g., $X^1$ or $X^2$, with a strong possibility meant for $H(T|X^1) \geq 0$ and $H(T|X^2) \geq 0$, and so, based on the view of KL-divergence, we reckon that $\mathcal{P}_{\{T-X^1\}}$ and $\mathcal{P}_{\{T-X^2\}}$ exist. For the KL-divergence, $\mathcal{P}_{X^1}$ or $\mathcal{P}_{X^2}$ is constant, and therefore it is very likely that $D_{KL}(\mathcal{P}_{X^1{\{T-X^2\}}}||\mathcal{P}_{X^1}\mathcal{P}_{\{T-X^2\}})$ or $D_{KL}(\mathcal{P}_{X^2{\{T-X^1\}}}||\mathcal{P}_{X^2}\mathcal{P}_{\{T-X^1\}})$ exists in like manner:\\
	\\$D_{KL}(\mathcal{P}_{X^1{\{T-X^2\}}}||\mathcal{P}_{X^1}\mathcal{P}_{\{T-X^2\}}) \geq 0$\\
	\\$D_{KL}(\mathcal{P}_{X^2{\{T-X^1\}}}||\mathcal{P}_{X^2}\mathcal{P}_{\{T-X^1\}}) \geq 0$\\
	\\$\therefore I(X^1;T|X^2) \geq 0$ and $I(X^2;T|X^1) \geq 0$\\
	\\$\therefore I(X^1;Y^*;T|X^2) \geq 0$ and $I(X^2;Y^*;T|X^1) \geq 0$\\
	\\$\therefore I(Y;T) \leq I(Y;T) + I(X^1;Y^*;T|X^2) + I(X^2;Y^*;T|X^1)$
	\\
\end{proof}

\subsection{How does GSWD implement nonlinear mapping?} \label{sec:nonlmap}
As mentioned in Section \ref{sec:gloco}, $GR_\vartheta$ represents one-dimensional nonlinear projection operation on the probability measure ${P _r}$ and ${P _g}$, which is defined as:
\begin{equation}
	G{R_\vartheta }{P_i}\left( x_i \right) = \int_{{\Sigma _i}} {{P_i}\left( x_i \right)} \delta \left( {t - ge\left( {x_i,\vartheta } \right)} \right)d{x_i}
\end{equation}
where $i \in \left\{ {r,g} \right\}$, $\delta \left( {\cdot} \right)$ is the one-dimensional Dirac delta function, $t \in \mathbb{R}$, and ${ge\left( {\cdot,\vartheta } \right)}$ is a pre-defined nonlinear function that must satisfy the following four conditions:
\begin{itemize}	
	\item ${ge\left( {\cdot,\vartheta } \right)}$ is a real-valued $C^\infty $ function.
	\item ${ge\left( {\cdot,\vartheta } \right)}$ is homogeneous of degree one in $\vartheta$, \textit{i.e.},
	\begin{equation}
		\forall \upsilon  \in R,{ge\left( {\cdot,\upsilon  \vartheta } \right)} = \upsilon {ge\left( {\cdot,\vartheta } \right)}
	\end{equation}
	\item ${ge\left( {\cdot,\vartheta } \right)}$ is non-degenerate in the sense that
	\begin{equation}
		\forall \vartheta  \in {\Omega _\vartheta }\backslash \left\{ 0 \right\},x \in X,\frac{{\partial ge}}{{\partial x}}\left( {x,\vartheta } \right) \ne 0
	\end{equation}
	\item The mixed Hessian of ${ge\left( {\cdot,\vartheta } \right)}$ is strictly positive, \textit{i.e.}
	\begin{equation}
		\det \left( {{{\left( {\frac{{{\partial ^2}ge}}{{\partial {x_i}\partial {\vartheta _j}}}} \right)}_{i,j}}} \right) > 0
	\end{equation}
\end{itemize}
${ge\left( {\cdot,\vartheta } \right)}$ is a nonlinear function so that the GSWD achieves to map high-dimensional representations to one-dimensional representations in a nonlinear manner.

\subsection{Extended comparisons} \label{sec:extended}
In this section, we conduct further experiments to study the intrinsic property of our proposed method.

\begin{figure}
	\vskip 0.1in
	\begin{center}
		\centerline{\includegraphics[width=0.95\columnwidth]{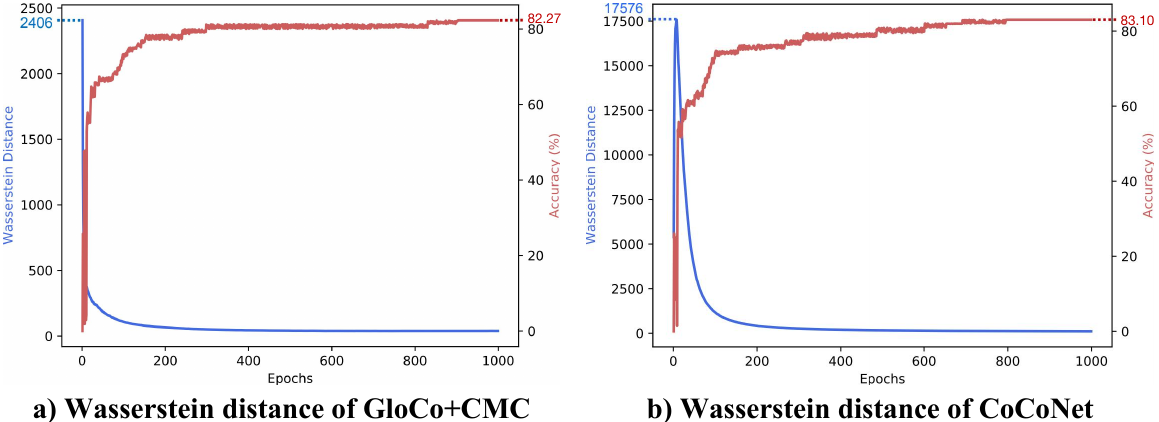}}
		\vskip 0in
		\caption{Extended verification of studying the Wasserstein distance changing trend properties of GloCo+CMC and CoCoNet in optimization on the CIFAR10 dataset.}
		\label{fig:wdrecord}
	\end{center}
	\vskip -0.35in
\end{figure}

\subsubsection{Study on the Wasserstein distance changing trends}
As $a)$ and $b)$ in Figure \ref{fig:wdrecord}, they show the changing trends of the sum of Wasserstein distances between views in optimization. In $a)$, it is based on GloCo+CMC, and the result of CoCoNet is shown in $b)$. We found that, although the Wasserstein distance in both $a)$ and $b)$ can reach convergence, the Wasserstein distance in $b)$ converges slower than in $a)$. Additionally, the Wasserstein distance's peak value in $b)$ is much higher than that in $a)$. The additional local complementarity preserving module affects the convergence of Wasserstein distance, and eventually, both the loss of the local complementarity preserving module and Wasserstein distance can converge. As such, the findings indicate that the game of simultaneously training the local complementarity and global consistency preserving modules is similar to the adversarial learning process, which helps CoCoNet to learn features with local complementarity and global consistency.

\begin{figure}
	\vskip 0.1in
	\begin{center}
		\centerline{\includegraphics[width=0.95\columnwidth]{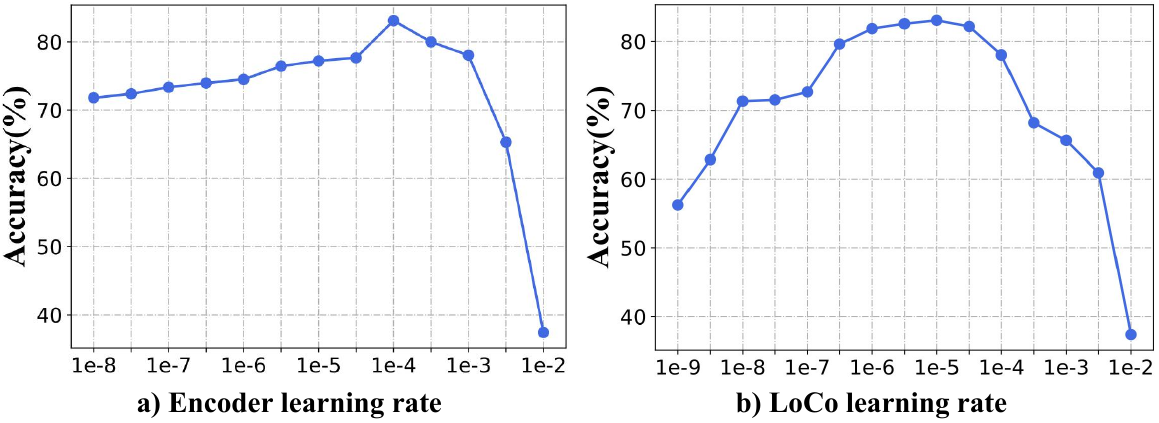}}
		\vskip 0in
		\caption{Extended verification of studying learning rates of the encoders and LoCo network in optimization on the benchmark CIFAR10 dataset.}
		\label{fig:lr}
	\end{center}
	\vskip -0.35in
\end{figure}

\subsubsection{Study on optimizations}
As manifested in Figure \ref{fig:lr}, we studied the learning rates of the encoders and the local complementarity preserving module (i.e., LoCo) respectively, while we excluded the learning rate parameter study of the global consistency preserving network (i.e., GloCo) from the experiment, because we found that the learning rate of the GloCo has little effect on the performance of the proposed method. Furthermore, the objective of GloCo is to calculate the Wasserstein distances between views, so the learning rate of GloCo does not enhance the accuracy of CoCoNet. To explore the influence of different parts of CoCoNet, we selected wide ranges for the target learning rates, e.g., a range of \{$10^{-9}, 10^{-8}, ..., 10^{-3}-, 10^{-2}$\} is for the uniform learning rate of the encoders, a range of \{$10^{-8}, 10^{-7}, ..., 10^{-3}, 10^{-2}$\} is for the learning rate of LoCo and fixed the other learning rates. We observed that the appropriate learning rates of the encoders and LoCo can promote the performance of our proposed method by a wide margin. Consider that the learning rate determines the step length of the weight iteration, so it is a very sensitive parameter. It has a significant effect on the model performance (e.g., the initial learning rate must have an optimal value. If it is too large, the model will not converge, and if it is too small, the model will converge slowly or fail to learn). Therefore, we concluded that the encoders and LoCo both have great impacts on the performance of CoCoNet.

\subsubsection{Study on the loss convergence}
From $a)$, $b)$, and $c)$ in Figure \ref{fig:lossconverge}, it can be found that all losses can reach convergence smoothly, which proves that the gradient descent of the loss of each part will not conflict with others in optimization. Moreover, it also verifies the integrity, robustness, and consistency of the proposed method.

We also conducted additional experiments to clarify the loss convergence of CoCoNet and the ablation model, i.e., LoCo. As in Figure \ref{fig:lossconverge}, plots $d)$ and $e)$ separately show the relationships between the total loss and the accuracy based on LoCo and CoCoNet. We can find that in both of the optimization processes of LoCo and CoCoNet models, the total losses will eventually converge, while CoCoNet will be slightly slower to reach convergence during the training process. Meanwhile, as demonstrated in the classification comparisons, CoCoNet outperforms LoCo, which indicates that the global consistency preserving module can indeed enhance the performance of the proposed method. Hence, with the addition of GloCo, the initial loss tends to increase, and this shows that the additional module allows CoCoNet to have greater optimization potential.

\end{document}